\documentclass[letterpaper]{article} % DO NOT CHANGE THIS
\usepackage{arxiv_style_new}  % DO NOT CHANGE THIS
\usepackage{times}  % DO NOT CHANGE THIS
\usepackage{helvet}  % DO NOT CHANGE THIS
\usepackage{courier}  % DO NOT CHANGE THIS
\usepackage[hyphens]{url}  % DO NOT CHANGE THIS
\usepackage{graphicx} % DO NOT CHANGE THIS
\usepackage{natbib}  % DO NOT CHANGE THIS AND DO NOT ADD ANY OPTIONS TO IT
\usepackage{caption} % DO NOT CHANGE THIS AND DO NOT ADD ANY OPTIONS TO IT
\usepackage{amsmath}
\usepackage{amssymb}
\usepackage{subcaption}
\usepackage{enumitem}
\usepackage{amsmath, amssymb, amsthm}
\usepackage{mathtools}
\mathtoolsset{showonlyrefs}
\usepackage{bm}
\usepackage{comment}
\usepackage{enumerate}
\usepackage{adjustbox}
\usepackage{booktabs}
\usepackage{siunitx} % for aligning numbers properly
\usepackage[dvipsnames]{xcolor}
\usepackage{pgffor} % for looping
\usepackage{tikz}

\usepackage{booktabs}

\usepackage{algorithm}
\usepackage{algpseudocode}
\usepackage{algorithmicx}
\usepackage{enumitem}
\newcommand{\revise}[1]{#1}

\usepackage{cite}
\newcommand{\R}{\mathbb{R}}
\newcommand{\N}{\mathbb{N}}

\newcommand{\Q}{\mathcal{Q}}
\newcommand{\Prob}{\mathcal{P}}

\newcommand{\A}{\operatorname{A}}

\renewcommand{\d}{\, \mathrm{d}}

\newcommand{\Sph}{\mathbb{S}}

\newcommand{\SW}{\operatorname{SW}}
\newcommand{\W}{\operatorname{W}}
\newcommand{\GW}{\operatorname{GW}}
\newcommand{\FGW}{\operatorname{FGW}}
\newcommand{\D}{\mathcal{D}}

\newcommand{\argmin}{\operatorname*{arg\,min}}

\newtheorem{theorem}{Theorem}[section]

\newtheorem{proposition}[theorem]{Proposition}

\theoremstyle{definition}

\newcommand{\XX}{\mathbb{X}}
\newcommand{\XXS}{\mathbb{X}^{\mathcal{S}}}
\newcommand{\YY}{\mathbb{Y}}
\newcommand{\YYS}{\mathbb{Y}^{\mathcal{S}}}
\newcommand{\T}{\mathcal{T}}
\newcommand{\FLB}{\operatorname{FLB}}
\newcommand{\SLB}{\operatorname{SLB}}
\newcommand{\TLB}{\operatorname{TLB}}
\newcommand{\STLB}{\operatorname{STLB}}
\newcommand{\FTLB}{\operatorname{FTLB}}
\newcommand{\SFTLB}{\operatorname{SFTLB}}
\newcommand{\LD}{\operatorname{LD}}

\newcommand{\vq}{\mathbf q}

\title{A Novel Sliced Fused Gromov-Wasserstein Distance}
\author{
    Moritz Piening\equalcontrib,
    Robert Beinert\equalcontrib
}
\affiliations{
    Institut für Mathematik, Technische Universität Berlin,
    Straße des 17. Juni 136, 10623 Berlin, Germany\\
    piening@math.tu-berlin.de, beinert@math.tu-berlin.de
}

\begin{document}

% ----- Title and Abstract -----

\maketitle

\begin{abstract}
    The Gromov--Wasserstein (GW) distance
    and its fused extension (FGW) 
    are powerful tools for comparing heterogeneous data.
    Their computation is, however, challenging
    since both distances are based on 
    non-convex, quadratic optimal transport (OT) problems.
    Leveraging 1D OT,
    a sliced version of GW has been proposed
    to lower the computational burden.
    Unfortunately,
    this sliced version is restricted to Euclidean geometry
    and loses invariance to isometries,
    strongly limiting its application in practice.
    To overcome these issues,
    we propose a novel slicing technique for GW as well as for FGW
    that is based on an appropriate lower bound, 
    hierarchical OT,
    and suitable quadrature rules for the underlying 1D OT problems.
    Our novel sliced FGW significantly reduces the numerical effort
    while remaining invariant to isometric transformations
    and allowing the comparison of arbitrary geometries.
    We show that our new distance actually defines 
    a pseudo-metric for structured spaces that bounds FGW from below
    and study its interpolation properties between sliced Wasserstein and GW.
    Since we avoid the underlying quadratic program,
    our sliced distance is numerically more robust and reliable
    than the original GW and FGW distance;
    especially in the context of shape retrieval 
    and graph isomorphism testing.
\end{abstract}

\section{Introduction}
The Gromov--Wasserstein (GW) distance
\cite{memoli2011gromov}
and the Fused Gromov--Wasserstein (FGW) distance
\cite{vayer2020fused}
extend the classical optimal transport (OT) 
framework \cite{villani2003topics}
to the comparison of heterogenous data by modelling them as metric measure spaces
(mm-spaces).
While the resulting distances are powerful tools,
their computation is costly
and does not admit an exact solution.
To accelerate the underlying
non-convex, quadratic program,
several algorithms have been proposed
using regularization techniques \cite{peyre2016gromov, PeyreCuturi2019},
employing linearized distances \cite{beier2022linear, nguyen2023linearfused}, quantizating the problem \cite{chowdhury2021quantized}  or
using constrained optimization \cite{scetbon2022lowrankgw, scetbon2023unbalanced}. 
As a consequence,
established solvers mainly compute upper bounds
and numerical approximations of the actual GW and FGW distance.
An alternative acceleration framework builds on 1D OT and random projections to define the so-called sliced GW distance \cite{vayer2019sgw,beinert2023assignment}. 
However, sliced GW is restricted to Euclidean geometry and lacks isometric invariance, limiting its practical use.
As a remedy, we design a novel sliced GW distance fixing these issues and extend it to FGW. In particular, our new distance allows for an analytical computation outlined in Figure~\ref{fig:idea_viz}.

For our design, 
we turn towards the complementary literature on lower GW bounds 
\cite{memoli2011gromov, vayer2019sgw,redko2020coot, jin2022orthogonal}.
Since many of these rely on optimization heuristics,  
such as bi-convex relaxation \cite{redko2020coot} 
or sorting \cite{vayer2019sgw, beinert2023assignment},
and we are interested in exact computation,
we turn our attention to the bounds in \cite{memoli2011gromov}.
In particular, we are motivated by recent theoretical work proving their metricity on certain spaces \cite{memoli2022distance}
and the practical effectiveness of such bounds in graph isomorphism testing \cite{weitkamp2022gromov, weitkamp2024distribution}, where established GW solvers fail in practice. 
The aim behind graph isomorphism testing \cite{shervashidze2011wlkernel, grohe2020graph}
is to verify or disprove
whether two or more graphs are topologically equivalent or not.
\begin{figure*}
    \centering
    \includegraphics[width=\linewidth]{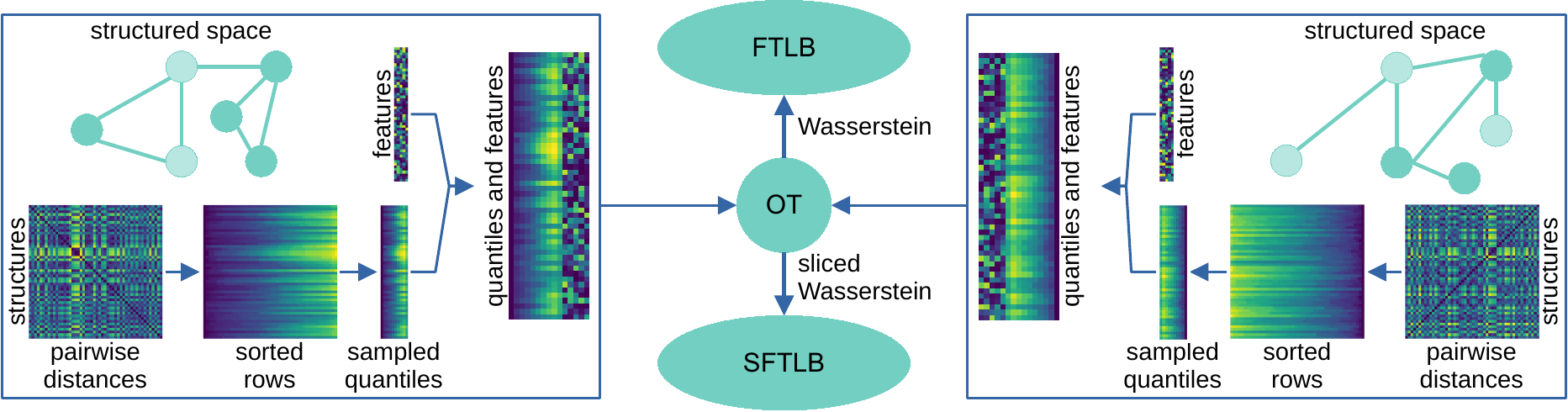}
    \caption{
    Visualization of our novel sliced fused Grovov--Wasserstein distance called SFTLB:
    Starting from two structured spaces like labeled graphs equipped with a graph distance,
    we sort the corresponding distance matrices row-wise.
    Subsequently, 
    we determine samples of the quantile functions of the local distance distribution for each node.
    Afterward,
    we concatenate the sampled quantiles with the original features
    and interpret the outcome as an empirical Euclidean measure.
    Finally, we compute the Wasserstein or sliced Wasserstein distance between these measures.}
    \label{fig:idea_viz}
\end{figure*}

In this paper, 
we focus on the strongest of M{é}moli's GW bounds%
---the third lower bound (TLB)---%
and start by extending TLB to the fused case. 
TLB itself is defined as a hierarchical OT problem, 
i.e., 
an OT problem whose costs are given by another OT problem.
Indeed, 
formulations of such kinds recently have attracted attention from the OT community
\cite{dukler2019wasserstein, alvarez2020geometric}. 
Despite interesting theoretical properties 
\cite{delon2020wasserstein, bonet2025flowing},
these problems are usually computationally challenging.
Therefore, recent work has focused on extending
sliced Wasserstein (SW) techniques to the hierarchical case for acceleration 
\cite{piening2025smw, nguyen2025sotdd}. 
Keeping this technique in mind, we tackle the computational burden of our new FGW lower bound by
rewriting it as an Euclidean Wasserstein distance.
To circumvent the limitation of slicing in high dimensions 
\cite{tanguy2025properties},
we propose to employ numerical quadrature schemes.
All in all,
we deduce an equivalent, but more efficient, lower bound.
Apart from extending the literature on sliced hierarchical OT,
this also overcomes previous limitations of sliced GW, 
where sorting-based solvers might fail \cite{vayer2019sgw, beinert2023assignment}.
We summarize our contributions as follows:
% ENUMERATE AND NO LINEBREAK THIS WAY TO AVOID ERROR!
\begin{itemize}
    \item[1.] {Towards the goal of an exact solver, we generalize Mémoli's GW bound to FGW and recall its connection to a classical Wasserstein problem.}
    \item[2.] {To overcome the limitations of sliced GW, we employ this connection and numerical quadrature schemes to define an effective novel sliced FGW distance between arbitrary structured spaces beyond Euclidean geometry.}
    \item[3.] {We prove that our new distance is a pseudo-metric interpolating between TLB and SW. Additionally, our proof illustrates the dimensional dependence of SW.}
    \item[4.] {Lastly, we provide numerical experiments for our sliced FGW, showcasing its efficiency as well as its applicability for shape comparison and graph isomorphism testing.}
\end{itemize}
\revise{Note that this is an extended version of a AAA'26 conference paper \cite{piening2025slicedgw_aaai}.}

% ----- Preliminary Definitions -----
\section{Optimal Transport-Based Distances}

At the heart of OT is the calculation of
optimal transfers from one measure to another.
For this,
let $\Prob(X)$ denote the space of Borel probability measures
on a compact Polish space $X$.
A measure $\xi \in\Prob(X)$ 
is translated to 
another compact Polish space $Y$
via a mapping $T\colon X \to Y$
by the \emph{push-forward} 
$T_\sharp \, \xi \coloneqq \xi \circ T^{{-}1}$.
Based on the canonical projections 
$\pi_\bullet$ onto a component,
e.g., $\pi_{X} \colon X \times Y \to X, (x, y) \mapsto x$,
the set of \emph{transport plans}
between $\xi \in \Prob(X)$, $\upsilon \in \Prob(Y)$ 
is defined as
\begin{equation}
    \Gamma(\xi, \upsilon) 
    \coloneqq
    \bigl\{ 
    \gamma \in\Prob(X \times Y) 
    \bigm\vert
    \pi_{X,\sharp} \, \gamma = \xi,
    \pi_{Y,\sharp} \, \gamma = \upsilon
    \bigr\}.
\end{equation}
In general,
OT seeks to find an optimal transport plan
minimizing a given loss function.
For certain losses,
this procedure yields a metric 
between probability measures.
    
\subsection{Wasserstein Distance}

For compact $Z \subset \R^d$
equipped with the Euclidean norm $\lVert \cdot \rVert$,
$\zeta, \zeta' \in \Prob(Z)$,
and $p \in [1, \infty)$,
the \emph{$p$-Wasserstein distance} is introduced as
\begin{equation}
    \label{eq:wasserstein}
    \W^p_p(\zeta, \zeta') 
    \coloneqq
    \smashoperator[l]{
    \inf_{\gamma \in\Gamma(\zeta, \zeta')}}
    \underbracket{\int_{Z \times Z} 
    \lVert z - z' \rVert^p 
    \, \d \gamma(z, z')}_{\eqqcolon \T_p(\gamma)}
\end{equation}
and defines a metric\footnote{\revise{Symmetric, positive definite function with triangle inequality.}} on $\Prob(Z)$. 
Its computation requires 
expensive optimization procedures
since closed-form solutions are rarely available.

One exception is the special case $Z \subset \R$,
where the optimal plan can be expressed
using the \emph{quantile function}:
\begin{equation}
    q_\zeta \colon (0,1) \to \R, \,
    s \mapsto 
    \inf \, \bigl\{ z \in \R 
    \bigm\vert
    \zeta((-\infty, z]) > s 
    \bigr\}.
\end{equation}
Then, the $p$-Wasserstein distance admits the closed-form
\begin{equation}
    \label{eq:1d_continuous_quantile_formula}
    \W_p^p(\zeta, \zeta') = \int_0^1 \big|q_\zeta(s) - q_{\zeta'}(s)\big|^p \, \d s;
\end{equation}
see \cite{villani2003topics}.
For discrete measures with with $n$ and $m$ support points,
the integral in \eqref{eq:1d_continuous_quantile_formula}
reduces to a sum with $n +m -1$ summands;
see \cite[Prop.~3.5]{PeyreCuturi2019} 
and \cite[Thm.~2.18]{villani2003topics}.

\subsection{Sliced Wasserstein Distance}
To leverage the computational benefits 
of the 1D Wasserstein distance,
we consider the \emph{slicing operator}
\begin{equation}{}
    \label{eq:projection}
    \pi_\theta 
    \colon
    \R^d \to \R
    ,\;
    z \mapsto \theta \boldsymbol{\cdot} z,
    \quad
    \theta \in\Sph^{d{-}1},
\end{equation}\normalsize
where 
$\Sph^{d-1} \coloneqq \{\theta \in \R^d \mid \lVert \theta \rVert = 1\}$
denotes the sphere with surface area $\A(\Sph^{d-1})$
and $\boldsymbol{\cdot}$ the Euclidean inner product.
For compact $Z \subset \R^d$, 
$\zeta,\zeta' \in \Prob(Z)$, and $p \in [1, \infty)$,
the \emph{sliced $p$-Wasserstein \emph{(SW)} distance} 
is introduced as
\begin{equation}
    \label{eq:sliced_wasserstein}
    \SW^p_p(\zeta, \zeta') 
    \coloneqq
    \frac{1}{\A(\Sph^{d{-}1})}
    \int_{\Sph^{d{-}1}}
    \W^p_p(\pi_{\theta,\sharp} \, \zeta, \pi_{\theta,\sharp} \, \zeta') 
    \d\theta
\end{equation}
with respect to the surface measure on $\Sph^{d-1}$.
The interest in this sliced distance is justified by 
the close relation to the classical Wasserstein distances.
Namely,
for compact $Z~{\subset}~\R^d$,
both distances display a form of metric equivalence,
i.e.,
there exists $C_Z>0$ such that
\begin{equation}{}
    \label{eq:sw_equivalnce}
    \SW^p_p (\zeta, \zeta')
    \le 
    \W^p_p (\zeta, \zeta')
    \le
    C_Z \SW^{{\frac{1}{(d{+}1)}}}_p(\zeta, \zeta')
\end{equation}\normalsize
for any $\zeta, \zeta' \in\Prob(Z)$;
see \cite{bonnotte2013unidimensional, bonneel2015slicedbarycenters, KPR16}.
Moreover, the spherical integral can be efficiently 
approximated via random projections
using Monte Carlo or Quasi-Monte Carlo methods 
for low-dimensional data
\cite{nguyen2024quasi, hertrich2025qmc_slicing}
as well as Gaussian approximations 
for high-dimensional data \cite{nadjahi2021fast}.
Note that 
there exist various extensions
to non-Euclidean spaces \cite{kolouri2019generalized, quellmalz2023sliced_optimal_transport, bonet2025sliced_hadamard}, invariant data \cite{beckmann2025max,beckmann2025normalized}, general probability divergences \cite{nadjahi2020statistical, hertrich2024generative}
and kernels \cite{rux2025slicing}.

\subsection{Gromov--Wasserstein Distance}
To compare measures on distinct metric spaces, 
the GW distance extends classical OT by matching pairwise distances.
For this,
let $\mathbb{X} \coloneqq (X,g,\xi)$ 
be a \emph{metric measure (mm) space} 
consisting of a compact metric space $(X,g)$
and $\xi \in \Prob(X)$.
For mm-spaces $\mathbb{X} \coloneqq (X,g,\xi)$ 
and $\mathbb{Y} = (Y,h,\upsilon)$,
the \emph{Gromov--Wasserstein \emph{(GW)} distance} is then defined as
\begin{equation}
    \begin{aligned}
        \GW^p_p(\XX,\YY)
        &\coloneqq
        \smashoperator[l]{
        \inf_{\gamma\in\Gamma(\xi,\upsilon)}}
        \iint_{(X\times Y)^2} 
        \bigl|g(x,x') - h(y,y')\bigr|^p
        \\[-9pt]
        &\hphantom{\coloneqq\smashoperator[l]{
        \inf_{\gamma\in\Gamma(\xi,\upsilon)}}}
        \underbracket{
            \hspace{60pt}\times
         \,\d\gamma(x',y')
        \,\d\gamma(x,y)}
        _{\eqqcolon \D_p(\gamma)}.
    \end{aligned}
    \label{eq:gw}
\end{equation}
Here, 
each transport plan $\pi$ induces 
an overall distortion $\D_p$ between $g$ and $h$.
This expression defines a metric on the equivalence classes of isomorphic mm-spaces;
see \cite{memoli2011gromov}. 
Since \eqref{eq:gw} is a non-convex, quadratic program,
its minimization typically relies on 
entropic regularization
and costly block-coordinate descent 
\cite{PeyreCuturi2019}. 
As a consequence, 
a set of hierarchical lower bounds has been proposed in
\cite{memoli2011gromov,memoli2022distance}.

\paragraph{First Lower Bound}
Define the pointwise \emph{$p$-eccentricity}
$s_{\XX,p} \colon X \to \R$
of an mm-space $\XX \coloneqq (X,g,\xi)$ as
\begin{equation}
    s_{\XX,p}^p(x)
    \coloneqq \int_X g^p(x,x')\,\d\xi(x').
\end{equation}
Then, the first lower GW bound between $\XX$ and $\YY$ is
\begin{equation}
\FLB^p_p(\XX,\YY)
\coloneqq \,
\W_p^p(s_{\XX,p,\sharp} \, \xi, s_{\YY,p,\sharp} \, \upsilon),
\end{equation}
which compares only the first‐order moments of pairwise distances.
This is an efficiently solvable 1D OT problem.

\paragraph{Second Lower Bound}
By matching the entire distribution of distances,
the second lower GW bound reads as
\begin{equation}
    \SLB^p_p(\XX,\YY)
    \coloneqq 
    \W_p^p\bigl(g_\sharp(\xi\otimes\xi), h_\sharp(\upsilon\otimes\upsilon)\bigr),
\end{equation}
where $\otimes$ denotes the product measure.
This lifts FLB by considering all pairwise distances at once. 
Again, 
characterization as 1D transport problem enables fast computation.

\paragraph{Third Lower Bound}
For each $(x,y) \in X \times Y$,
the local distance distributions
between $\XX$ and $\YY$ is defined via
\begin{align}
    \LD^p_{p}(x,y)
    &\coloneqq 
    \smashoperator[l]{
    \inf_{\gamma \in\Gamma(\xi,\upsilon)}}
    \int_{X \times Y}
    \bigl\lvert g(x, x') - h(y, y') \bigr\rvert^p
    \,\d \gamma(x',y')
    \\[-10pt]
    &= 
    \W^p_p\bigl( g(x,\cdot)_\sharp \, \xi, h(y,\cdot)_\sharp \,\upsilon\bigr).
    \label{eq:local_distance_distribution_wasserstein}
\end{align}
Decoupling the minimization of the inner and outer integral in \eqref{eq:gw},
we arrive at the third lower GW bound:
\begin{equation}
    \label{eq:TLB}
    \TLB^p_p(\XX,\YY)
    \coloneqq 
    \smashoperator[l]{
    \inf_{\gamma \in\Gamma(\xi,\upsilon)}}
    \underbracket{
    \int_{X\times Y}
    \LD^p_{p}(x,y)
    \,\d\gamma (x,y)}
    _{\eqqcolon \tilde\D_p(\gamma)}.
\end{equation}
Note that 
the inner plans
in \eqref{eq:local_distance_distribution_wasserstein}
do not match 
the outer plan 
in \eqref{eq:TLB}.
Similar to the other two cases, 
the estimation of LD
benefits from accelerated 1D solutions.
In contrast, 
TLB still requires a costly OT minimization.

\smallskip

Comparing all three lower bounds,
we obtain the hierarchy
\cite{memoli2011gromov,chowdhury2019gromov}:
\begin{equation}
    \FLB_p(\XX{,}\YY)  
    \le
    \SLB_p(\XX{,}\YY) 
    \le
    \TLB_p(\XX{,}\YY) 
    \le
    \GW_p(\XX{,}\YY).
\end{equation}
Moreover,
TLB is positive definite
and thus defines a metric
for certain classes of mm-spaces, 
whereas FLB and SLB fail 
to distinguish isomorphic mm-spaces 
\cite{memoli2022distance}.
To avoid the outer OT problem in \eqref{eq:TLB},
\cite{sato2020fast} rely on the maximum mean discrepancy (MMD)
with respect to the 1D Wasserstein distance as kernel,
which leads to an efficient computation scheme.
In this paper,
we follow a different approach by
accelerating the computation of the outer OT problem
using an efficient slicing,
which actually leads to a metric equivalence with TLB.
%\textcolor{red}{Note that we do not consider the sweep-line algorithm \cite{sato2020fast} proposed accelerated computation via a sweep-line algorithm.}

\subsection{Fused Gromov--Wasserstein Distance}
The GW distance allows
for a geometrically meaningful comparison 
between unlabeled, undirected graphs.
However,
due to the importance of node features in practical applications, 
practitioners often resort to a combination of Wasserstein and GW distances. 
For this,
the mm-spaces are extended by
a common compact feature space $Z \subset \R^d$.
More precisely,
a \emph{structured space}
$\XXS \coloneqq (X \times Z, g, \xi)$
consists of
a compact metric space $(X,g)$
and $\xi \in \Prob(X \times Z)$.
Structured spaces are
also known as labelled spaces or structured objects
\cite{vayer2020fused}. 
The space of all structured spaces 
is denoted
as $\mathcal{S}(Z)$.
Having the generalization of TLB 
and its slicing in mind,
we study the following distance on $\mathcal S(Z)$,
which is a specific variant of \cite[Def.~8]{vayer2020fused}:
for $\XXS \coloneqq (X \times Z, g, \xi)$,
and $\YYS \coloneqq (Y \times Z', h, \upsilon)$
with $Z = Z'$
as well as
$\alpha \in [0,1]$
and $p \in [1,\infty)$,
the \emph{fused GW \emph{(FGW)} distance} reads as
\begin{align}
    \FGW_{\alpha,p}^p (\XXS, \YYS)
    &\coloneqq
    \smashoperator[l]{
    \inf_{\gamma\in\Gamma(\xi,\upsilon)}}
    (1-\alpha)
    \,\D_{p}(\pi_{X\times Y, \sharp} \, \gamma)
    \\[-7pt]
    &\hspace{51pt}
    + \alpha \, \T_p(\pi_{Z \times Z', \sharp} \,\gamma),
    \label{eq:fgw}
\end{align}
where $\D_{p}$ denotes 
the distortion in \eqref{eq:gw},
$\T_p$ refers to the transport in \eqref{eq:wasserstein},
and $\alpha$ balances structure versus features. 
In the case $\alpha = 0$,
we recover the original GW distance of the structure part;
for $\alpha = 1$, 
we obtain the original Wasserstein distance of the label part.
Notice that,
for $\alpha \in (0,1)$
and $p \in [1, \infty)$,
FGW only defines a semi-metric\footnote{\revise{Metric without triangle inequality.}}
on the space of isomorphic structure spaces,
\revise{fulfilling the relaxed triangle inequality}
%\revise{i.e.,
%the triangle inequality is not in force,
%but here holds
%in the relaxed version}
\begin{align}
    &\FGW_{\alpha,p}^p(\YYS_1, \YYS_2)
    \\
    &\le
    2^{p-1} \bigl(
    \FGW_{\alpha,p}^p(\YYS_1, \XXS) 
    +
    \FGW_{\alpha,p}^p(\XXS,\YYS_2)
    \bigr);
    \label{eq:rel-tri}
\end{align}
see \cite[Prop.~4]{vayer2020fused}.

\section{Accelerating the Calculation of FGW}

\subsection{A Lower Bound Using TLB}

In principle,
any of the lower bounds for the GW distance
may be directly extended to the FGW distance.
Due to the metric properties of TLB
and its tight affinity to the original GW distance,
we solely focus on this bound.
The basic idea is to replace the GW distortion in FGW by \revise{TLB}.
In this manner,
for $\XXS \coloneqq (X \times Z, g, \xi)$, 
$\YYS \coloneqq (Y \times Z', h , \upsilon)$
with compact $Z = Z' \subset \R^d$,
$\alpha \in [0,1]$, and $p \in [1, \infty)$,
we introduce the \emph{Fused Third Lower Bound} (FTLB) via 
\begin{align}
    \FTLB^p_{\alpha, p}(\XXS,\YYS)
    &\coloneqq 
    \smashoperator[l]{
    \inf_{\gamma\in\Gamma(\xi,\upsilon)}}
    (1-\alpha) \,
    \tilde\D_{p}(\pi_{X \times Y, \sharp} \, \gamma)
    \\[-7pt]
    &\hspace{51pt}
    + \alpha \, \mathcal{T}_p(\pi_{Z \times Z', \sharp} \, \gamma),
    \label{eq:FTLB}
\end{align}
where $\tilde\D_p$ relates to the mm-spaces
$\XX \coloneqq (X,g,\pi_{X, \sharp} \, \xi)$
and $\YY \coloneqq (Y, h, \pi_{Y, \sharp} \, \upsilon)$.
Similar to GW and FGW,
this includes TLB for $\alpha = 0$.
Consequently,
all subsequent results apply to the classical GW distance as well.
Indeed, 
FTLB provides a lower bound
since we minimize over
multiple, uncoupled inner plans 
as opposed to finding a single coupled plan.
Due to the similarity to FGW,
many properties of the original distance carry over to FTLB.

\begin{proposition}
    \label{prop:ftlb_prop}
    For $\XXS \coloneqq (X \times Z, g, \xi)$, 
    $\YYS \coloneqq (Y \times Z, h, \upsilon)$, 
    $Z \subset \R^d$ compact,
    let $\xi^{\mathcal F} \coloneqq \pi_{Z, \sharp} \, \xi$, 
    $\upsilon^{\mathcal F} \coloneqq \pi_{Z, \sharp} \, \upsilon$
    and $\XX \coloneqq (X, g, \pi_{X,\sharp} \, \xi)$,
    $\YY \coloneqq (Y, h, \pi_{Y,\sharp} \, \upsilon)$.
    \begin{itemize}
        \item[\upshape(i)] There exists $\gamma^* \in \Gamma(\xi, \upsilon)$ minimizing \eqref{eq:FTLB}.
        \item[\upshape(ii)] $\FTLB_{\alpha, p}(\XXS,\YYS) \leq \FGW_{\alpha, p}(\XXS,\YYS)$.
        \item[\upshape(iii)] $\FTLB_{\alpha, p}(\XXS,\YYS) \to \W_p(\xi^{\mathcal F}, \upsilon^{\mathcal F})$ as $\alpha\to1$.
        \item[\upshape(iv)] $\FTLB_{\alpha,p} (\XXS, \YYS) \to \TLB_p(\XX,\YY)$ as $\alpha \to 0$.
        \item[\upshape(v)] $\FTLB_{\alpha, p}$ defines a pseudo-semi-metric%
            %\footnote{\revise{The metric is not definite (pseudo),
            %and the classical triangle inequality does not hold (semi).}}
            \footnote{\revise{Semi-metric $d$ without definiteness ($d(x,x') = 0 \not\Rightarrow x = x'$).}}
            on $\mathcal{S}(Z)$
            with
            \begin{equation}
            \begin{aligned}
                &\FTLB_{\alpha,p}^p(\YYS_1, \YYS_2)
                \\
                &\le
                2^{p-1} \bigl(
                \FTLB_{\alpha,p}^p(\YYS_1, \XXS) 
                +
                \FTLB_{\alpha,p}^p(\XXS,\YYS_2)
                \bigr)
            \end{aligned}
            \end{equation}
            for all $\XXS, \YYS_1, \YYS_2 \in \mathcal S(Z)$.
    \end{itemize}
\end{proposition}

\subsection{A Novel Lower Bound via Slicing}
While lowering the computational burden, 
FTLB still requires solving an OT problem. 
For two discrete measures of cardinality $n$, 
this leads to a computational complexity of $\mathcal{O}(n^3)$ 
for an exact OT solution 
and $\mathcal{O}(n^2 \log n)$ for a regularized one \cite{PeyreCuturi2019}.
To achieve further acceleration, 
we aim to combine ideas from SW 
and numerical integration. 
For this,
let $\XXS \coloneqq (X \times Z, g, \xi)$
and $\YYS \coloneqq (Y \times Z', h, \upsilon)$
be equipped with discrete measures
\begin{equation}
    \xi 
    \coloneqq 
    \sum_{i=1}^n \xi_i \, \delta_{(x_i, z_i)} 
    \quad\text{and}\quad
    \upsilon 
    \coloneqq 
    \sum_{j=1}^m \upsilon_j \, \delta_{(y_j, z'_j)}.
\label{eq:discrete_fused_measures}
\end{equation}

For the special case $p=2$,
we start with considering $\LD_2$ 
in \eqref{eq:local_distance_distribution_wasserstein},
which is the key ingredient 
of the structure distortion \smash{$\tilde\D_2$}.
The required 1D Wasserstein distance may here be analytically determined
by exploiting \eqref{eq:1d_continuous_quantile_formula},
where the quantile functions can be explicitly computed
using efficient sorting algorithms.
To surmount the outer OT problem of FTLB,
we instead propose to merely approximate 
\eqref{eq:1d_continuous_quantile_formula}
via some quadrature scheme,
i.e.,
\begin{equation}
    \label{eq:LD}
    \LD^2_{2}(x,y) 
    \approx 
    \sum_{k=1}^r
    w_k \, \bigl\lvert 
    q_{g(x,\cdot)_\sharp \xi}(s_k) 
    -
    q_{h(y,\cdot)_\sharp \upsilon}(s_k)\bigl\lvert^2
\end{equation}
with weights $w_k > 0$, 
knots $s_k \in (0,1)$,
and small $r \in \N$.
Note that 
we restrict ourselves to quadrature rules with positive weights.
Using 
\smash{$\vq_{\xi,x} \coloneqq (\sqrt w_k q_{g(x,\cdot)_\sharp \xi} (s_k))_{k=1}^r$}
and
\smash{$\vq_{\upsilon,y} \coloneqq (\sqrt w_k q_{h(y,\cdot)_\sharp \upsilon} (s_k))_{k=1}^r$},
we may also write
\begin{equation}
    \LD_2^2(x,y)
    \approx
    \lVert \vq_{\xi,x} - \vq_{\upsilon,y} \rVert^2.
\label{eq:ld_discrete_quantile}
\end{equation}
In this manner,
FTLB can be approximated by
\begin{align}
    &\FTLB_{\alpha,2}^2(\XXS, \YYS)
    \approx
    \smashoperator[l]{
    \inf_{\gamma \in \Gamma(\xi,\upsilon)}} 
    \smashoperator{
    \sum_{i,j=1}^{n,m}}
    \gamma_{ij}  
    \left\lVert
    \begin{smallmatrix}
        \sqrt{1 - \alpha} \, (\vq_{\xi,x_i} - \vq_{\upsilon, y_j})
        \\
        \sqrt \alpha \, (z_i - z'_j)
    \end{smallmatrix}
    \right\rVert^2
    \\[-30pt]
    &\hspace{8mm}=
    \W_2^2
    \Bigl(
    \underbracket{
    \sum_{i=1}^n \xi_i \,
    \delta_{\!\!\bigl(
    \begin{matrix}
        \scriptscriptstyle \sqrt{1-\alpha} \vq_{\xi,x_i}
        \\[-4pt]
        \scriptscriptstyle\sqrt\alpha z_i
    \end{matrix}\bigr)}
    }_{\eqqcolon \xi_\alpha^\Q \revise{\in \Prob(\R^{r+d})}},
    \underbracket{
    \sum_{j=1}^m \upsilon_j \, 
    \delta_{\!\!\bigl(
    \begin{matrix}
        \scriptscriptstyle \sqrt{1-\alpha} \vq_{\upsilon,y_j}
        \\[-4pt] 
        \scriptscriptstyle\sqrt\alpha z'_j
    \end{matrix}\bigr)}
    }_{\eqqcolon \upsilon_\alpha^\Q \revise{\in \Prob(\R^{r+d})}}
    \Bigr),
    \label{eq:approx-FTLB}
\end{align}
where $\pi_{ij} \coloneqq \pi((x_i, z_i), (y_j, z'_j))$.
This approximation carries over to TLB.
More precisely,
for $\XX \coloneqq (X, g,\xi)$
and $\YY \coloneqq (Y, h, \upsilon)$
with $\xi \coloneqq \sum_{i=1}^n \xi_i \, \delta_{x_i}$
and $\upsilon \coloneqq \sum_{j=1}^m \upsilon_j \, \delta_{y_j}$,
we here obtain
\begin{equation}
    \TLB_{2}^2(\XX,\YY)
    \approx
    \W_2^2
    \Bigl(
    \underbracket{
    \sum_{i=1}^n \xi_i \,
    \delta_{\vq_{\xi,x_i}}
    }_{\eqqcolon \xi^\Q \revise{\in \Prob(\R^{r})}},
    \underbracket{
    \sum_{j=1}^m \upsilon_j \, 
    \delta_{\vq_{\upsilon,y_j}}
    }_{\eqqcolon \upsilon^\Q \revise{\in \Prob(\R^{r})}}
    \Bigr).
    \label{eq:approx-TLB}
\end{equation}

Up to here,
we have approximated TLB and FTLB 
using a quadrature rule 
for the inner 1D Wasserstein distances.
To overcome the remaining outer OT minimization
behind the Wasserstein distances in 
\eqref{eq:approx-FTLB} and \eqref{eq:approx-TLB},
we finally propose
to replace these with the SW distance:
\begin{align}
    \SFTLB_{\alpha, 2}(\XXS,\YYS) 
    &\coloneqq 
    \SW_2(\xi^\Q_{\alpha}, \upsilon^\Q_{\alpha})
    \\
    \STLB_{2}(\XX,\YY) 
    &\coloneqq 
    \SW_2(\xi^\Q, \upsilon^\Q).
    \label{eq:sftlb_def}
\end{align}
In general,
the sliced lower bounds hold only approximately,
i.e.,
$\SFTLB_{\alpha,2}(\XXS, \YYS) \lessapprox \FGW(\XXS, \YYS)$,
where the quality depends on the employed quadrature rule.
Since $\xi^\Q_\alpha, \upsilon_\alpha^\Q \in \Prob(\R^{r+d})$
and $\xi^\Q, \upsilon^\Q \in \Prob(\R^r)$,
the employed SW distances can be efficiently computed
independently of the atom numbers $n$ and $m$ 
in \eqref{eq:discrete_fused_measures}.
Moreover,
many properties of the original FGW distance 
from Proposition~\ref{prop:ftlb_prop}
carry over to the sliced variant.

\begin{proposition}
    \label{prop:prop_sftlb_properties}
    For $\XXS \coloneqq (X \times Z, g, \xi)$, 
    $\YYS \coloneqq (Y \times Z, h, \upsilon)$, 
    $Z \subset \R^d$ compact,
    let $\xi^{\mathcal F} \coloneqq \pi_{Z, \sharp} \, \xi$, 
    $\upsilon^{\mathcal F} \coloneqq \pi_{Z, \sharp} \, \upsilon$
    and $\XX \coloneqq (X, g, \pi_{X,\sharp} \, \xi)$,
    $\YY \coloneqq (Y, h, \pi_{Y,\sharp} \, \upsilon)$.
    For $k,\ell \in \N$ with $\ell > 1$,
    let
    \begin{equation}
        \label{eq:discount_constant}
        c_{k, \ell} 
        \coloneqq 
        \sqrt{\frac{\A(\Sph^{k + \ell + 1})}{\A(\Sph^{\ell + 1})}
        \frac{\A(\Sph^{\ell - 1})}{\A(\Sph^{k + \ell - 1})}}.
    \end{equation}
    For an arbitrary quadrature rule, it holds:
     \begin{itemize}
     \item[\upshape(i)] $\SFTLB_{\alpha, 2}$ defines a pseudo-metric%
     %\footnote{Definiteness is not in force.}
     $\mathcal{S}(Z)$,
     \item[\upshape(ii)] $\SFTLB_{\alpha, 2}(\XXS,\YYS) 
        \to c_{r, d} \SW_2(\xi^{\mathcal F}, \upsilon^{\mathcal F})$
        as $\alpha \to 1$,
    \item[\upshape(iii)] $\SFTLB_{\alpha, 2}(\XXS,\YYS) 
        \to c_{d, r} \STLB_2(\XX, \YY)$ 
        as $\alpha \to 0$.
    \end{itemize}
\end{proposition}

Given that
$c_{k,\ell} \to 0$
for $k \to \infty$ and $\ell$ fixed 
\cite{gray2004tubes}, 
the constants $c_{r,d}$ and $c_{d,r}$ illustrate
the curse of dimensionality of
the SW distance. 
In contrast to classical Wasserstein transport, 
the value of the SW distance
for measures living on a subspace
depends 
on the dimension of the underlying Euclidean space.
Indeed, as we believe this to be of independent interest to the OT community,
we state this dependence as a general 
result employed in the proof of Proposition~\ref{prop:prop_sftlb_properties}.

\begin{proposition}
\label{prop:dimensional_dependence_sliced}
    For $\zeta, \zeta' \in\mathcal{P}(A \times Z)$,
    $A \subset \R^r$ and  $Z \subset \R^d$ compact,
    with $\pi_{A, \sharp} \, \zeta = \pi_{A, \sharp} \, \zeta' = \delta_0$,
    it holds 
    \begin{equation}
        \SW_2(\zeta, \zeta') 
        = 
        c_{r, d}  
        \SW_2(\pi_{Z, \sharp}\, \zeta, \pi_{Z, \sharp}\, \zeta').
    \end{equation}
\end{proposition}

Importantly,
for empirical measures $\xi$ and $\upsilon$ 
with $n=m$ and $\xi_i = \upsilon_i = 1 / n$
in \eqref{eq:discrete_fused_measures},
the rather simple, equispaced midpoint rule with $r=n$
becomes exact,
i.e.,
\begin{equation}
    \label{eq:exact_tlb_empirical}
    \FTLB_{\alpha, 2}(\XXS,\YYS)
    =
    \W_2(\xi^{\Q}_{\alpha}, \upsilon^{\Q}_{\alpha})
\end{equation}
instead of the approximation in \eqref{eq:approx-FTLB}.
Consequently,
SFTLB becomes 
an actual lower bound for the FGW distance.
Moreover, 
SFTLB and FTLB are metrically equivalent.

\begin{proposition}
    \label{prop:metr-equi}
    For $\XXS \coloneqq (X \times Z, g, \xi)$ \revise{and}
    $\YYS \coloneqq (Y \times Z, h, \upsilon)$
    \revise{with bounded $X$ and $Y$,
    i.e. $g(x,x') \le D$ and $h(y,y') \le D$ for fixed $D$, 
    and compact} $Z \subset \R^d$,
    let $\xi$ and $\upsilon$ in \eqref{eq:discrete_fused_measures} be empirical measures
    with $n = m$ and $\xi_i = \upsilon_i = 1 / n$.
    For the equispaced midpoint rule with $r = n$,
    it holds:
    \begin{itemize}
        \item[\upshape(i)] $\SFTLB_{\alpha,2}(\XXS, \YYS) \le \FGW_{\alpha,2}(\XXS, \YYS)$,
        \item[\upshape(ii)] there exists \revise{$C_{Z,D,r}>0$} such that
        \begin{align}
            \SFTLB_{\alpha, 2}(\XXS,\YYS) 
            \leq 
            \FTLB_{\alpha, 2}(\XXS,\YYS)
            \hphantom{.}
            &
            \\
            \leq 
            \revise{C_{Z,D,r}}  \SFTLB^{\frac{1}{r {+} d {+}1}}_{\alpha, 2}(\XXS,\YYS).
            &
            \label{eq:pseudo_metric_equivalence}
        \end{align}
    \end{itemize}
\end{proposition}

Note that 
\revise{$C_{Z,D,r}$} can be chosen 
such that 
\eqref{eq:pseudo_metric_equivalence} holds true
for all $\XXS$ and $\YYS$ of the stated form.
Furthermore,
Proposition~\ref{prop:metr-equi} can be directly generalized
to empirical measures $\xi$ and $\upsilon$
with different atom numbers $n$ and $m$
by using a non-equispaced midpoint rule
with $r = n + m - 1$,
\revise{see the remark in our technical appendix.}

\section{Numerical Experiments}

In the following,
we study the computational runtime of our distances,
the underlying transport plans,
as well as
their behaviour in applications like
free-support barycenters, 
shape classification,
and graph isomorphism testing.
All experiments\footnote{Code: \hspace{-.3mm}\url{github.com/MoePien/slicing_fused_gromov_wasserstein}}
were conducted with 13th Gen Intel Core i5-13600K CPU and an NVIDIA GeForce
RTX 3060 GPU (12 GB).

\subsection{Computational and Runtime Analysis}

For simplicity,
we assume that
the measures
of the structured spaces $\XXS \coloneqq (X \times Z, g, \xi)$
and $\YYS \coloneqq (Y \times Z, h, \upsilon)$
are empirical with the same number of atoms,
i.e.,
$n=m$ and $\xi_i = \upsilon_j = 1/n$ in \eqref{eq:discrete_fused_measures}.
The first step of FTLB and SFTLB 
is to sample the quantiles of the local distance distributions,
i.e., to compute $\vq_{\xi,x_i}$ and $\vq_{\upsilon, y_j}$ in \eqref{eq:approx-TLB},
which essentially requires a row-wise sorting
of the distance matrices regarding $g$ and $h$
with total complexity $\mathcal{O}(n^2 \log n)$.

\paragraph{FTLB}
In its second step,
FTLB requires solving the outer OT problem in \eqref{eq:approx-TLB}.
Calculating the OT costs requires $\mathcal O((r+d) \,n^2)$ operations.
The OT problem itself may be solved exactly in $\mathcal{O}(n^3)$ 
or approximately in $\mathcal{O}(n^2 \log n)$ using entropic regularization;
see \cite{PeyreCuturi2019}.
This leads to an all-in-all complexity 
of $\mathcal O((r+d + n) \, n^2)$ (exact)
or $\mathcal O((r+d+\log n) \, n^2)$ (approximated). 
In the experiments, we perform exact inner OT calculations with $r=n$.

\paragraph{SFTLB}
Here,
we initially determine $\xi^\Q_\alpha$ and $\upsilon^\Q_\alpha$.
Next,
we employ a Monte Carlo scheme to calculate \eqref{eq:sftlb_def} by
\begin{equation}
\label{eq:monte_carlo_sw}
    \SFTLB^2_{\alpha, 2}(\XXS, \YYS)
    \approx 
    \frac1L \sum_{\ell=1}^L
     \W^2_2(\pi_{\theta_\ell,\sharp} \, \xi_\alpha^\Q, \pi_{\theta_\ell,\sharp} \, \upsilon^\Q_\alpha)
\end{equation}
with uniformly distributed directions $\theta_\ell \in \Sph^{r+d-1}$.
The 1D Wasserstein distance can be computed by sorting,
yielding an all-in-all complexity of $\mathcal O(Ln(r+d) + (Ln + n^2) \log n)$, where we aim for a small $r$.
Here, the Monte Carlo scheme converges in $\mathcal{O}(\sqrt{L})$ \cite{nadjahi2020statistical}. 

\medskip

The computational complexity of FTLB and SFTLB is dominated by the preliminary row-wise sorting
and the employed OT solvers.
To study the actual run-time speed-up, 
we generate 5 pairs of $n\times n$ distance matrices with $n\in \{100, 1000, 10\,000\}$ and 1D normal features. 
Using the POT toolbox \cite{flamary2021pot} for the OT computations, 
we then determine the corresponding FTLB, SFTLB, and FGW distance.
The runtimes are recorded in Table~\ref{tab:time_tab},
where the significant speed-up due to the slicing is clearly visible.

\begin{table}
    \centering
    %\footnotesize
  \fontsize{8}{7.5}\selectfont
    \begin{tabular}{lccc}
    \toprule
    $n$ & 100 & 1000  & 10\,000 \\
    \midrule
    FGW & 0.03 $\pm$ 0.02 & 0.45 $\pm$ 0.07 & 178.29 $\pm$ 5.22\\
    FTLB & 0.06 $\pm$ 0.03 & 0.10 $\pm$ 0.02 & 15.69 $\pm$ 0.43\\
    SFTLB & 0.04 $\pm$ 0.04 & 0.03 $\pm$ 0.01 & 1.18 $\pm$ 0.00 \\
    \bottomrule
    \end{tabular}
    \caption{Average computation time in seconds for 5 instances with graph sizes 100, 1000, and 10\,000.
    The remaining parameters have been chosen as $r=10$ and $L=50$.}
    \label{tab:time_tab}
\end{table}

\subsection{FTLB Transport Plans}

\begin{figure}[t] 
    \centering
    \begin{subfigure}[b]{0.495\linewidth} 
        \centering 
        \includegraphics[width=\linewidth]{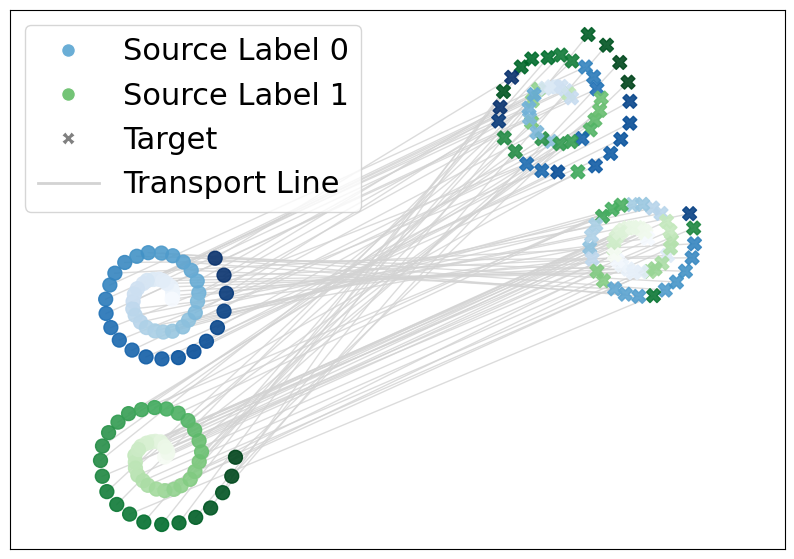}
        \caption{TLB}
        \label{fig:custom_plan_0.0}
    \end{subfigure}
    \hfill
    \begin{subfigure}[b]{0.495\linewidth} 
        \centering
        \includegraphics[width=\linewidth]{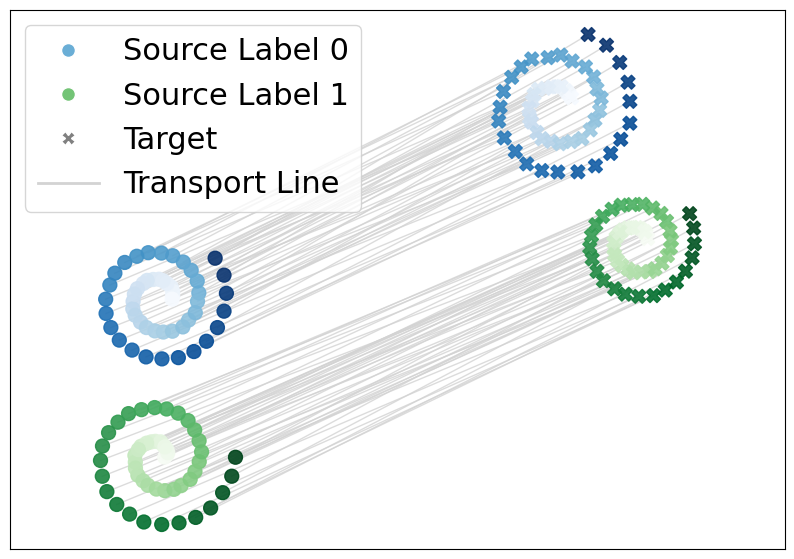}
        \caption{FTLB ($\alpha=0.9$)}
        \label{fig:custom_plan_0.9}
    \end{subfigure}
    \caption{Comparison of TLB and FTLB transport plans.
    Crosses are colored based on the transported mass.
    FTLB achieves a more regular transport plan by integrating labels.} 
    \label{fig:TLB_VS_FLTB_plan}
\end{figure}

Although the GW distance is well fitted 
to detect similarities between two single structures,
it usually fails to detect these in a set of independent structures.
The FGW distance can overcome this issue 
by assigning appropriate labels.
This observation carries over
to the outer transport plans of TLB and FTLB,
which is illustrated in Figure~\ref{fig:TLB_VS_FLTB_plan},
where the left two spirals are transported to the right two.
Source and target are here modelled as point clouds in $\R^2$
equipped with the Euclidean distance and binary feature labels.

\begin{figure}[t]
    \centering
    \includegraphics[width=.9\linewidth, trim={0 35. 0 6.}, clip]{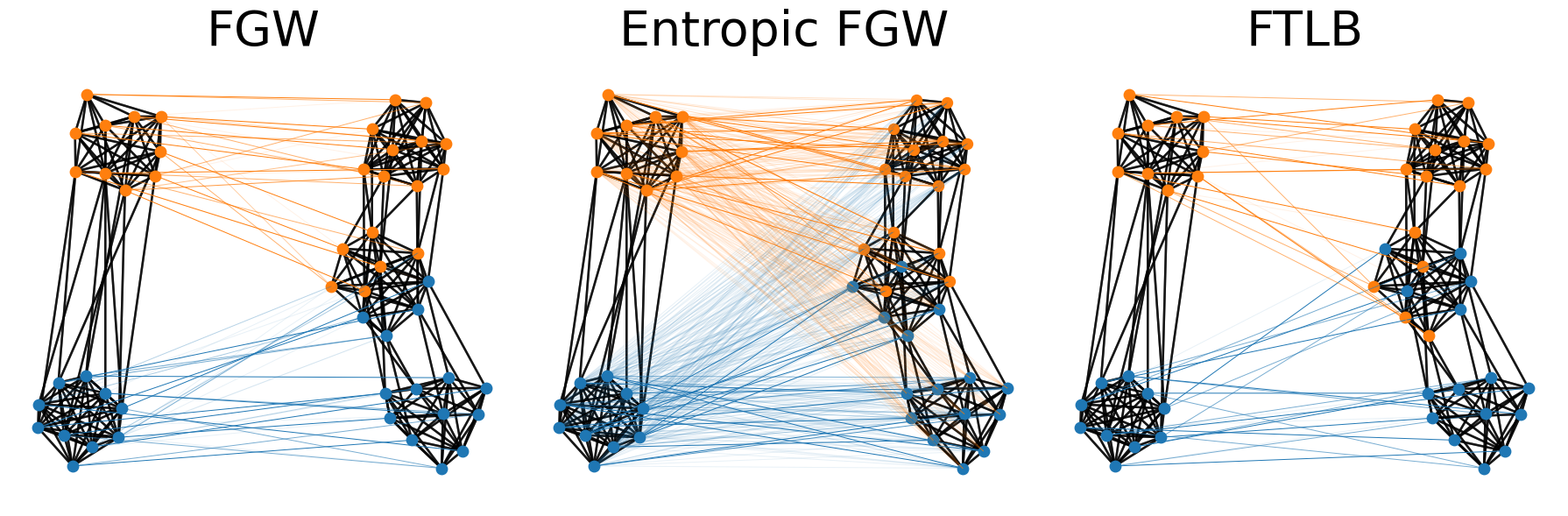}
    \caption{Comparison of FGW, entropic FGW ($\epsilon=0.01$), FTLP transport plan for $\alpha=0.5$.
    Target nodes are colored according to the transported mass.
    FGW and FTLB capture the same structural similarities.}
    \label{fig:GW_VS_FTLB_plan}
\end{figure}

To compare our FTLB transport with FGW and its entropic version,
we adapt an example from \cite{flamary2021pot};
see Figure~\ref{fig:GW_VS_FTLB_plan}. 
Each image displays the source graph on the left and the target graph on the right. 
The transports are visualized using the color of the community assignment in the source.
In this example,
the FTLB plan captures the main structures similarly to the original FGW plan,
indicating that our FTLB transport may be an efficently-to-compute alternative in practice.

\subsection{Free-Support Euclidean Barycenters}

\begin{figure}[t]
  \centering
    \begin{subfigure}[t]{\linewidth}
  \begin{subfigure}[b]{.4\textwidth}
    \includegraphics[width=\textwidth]{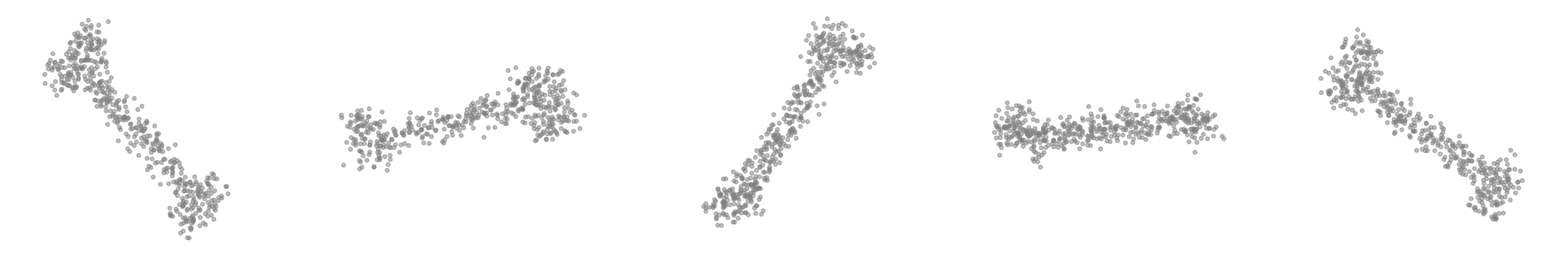}
    \caption*{Bone Inputs}
  \end{subfigure}
  \hfill
  \begin{subfigure}[b]{.15\textwidth}
    \includegraphics[width=\textwidth]{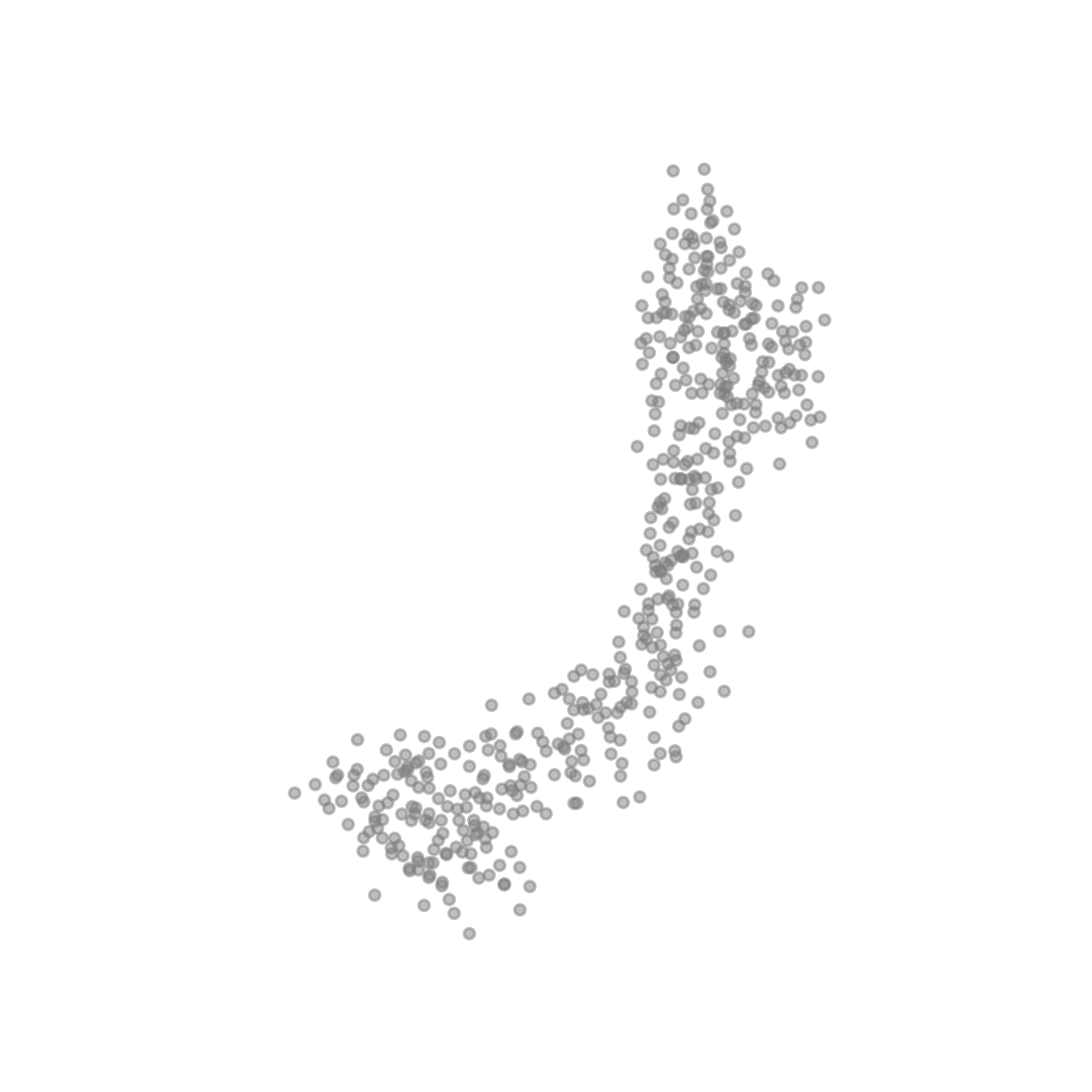}
    \caption*{TLB}
  \end{subfigure}
  \hfill
  \begin{subfigure}[b]{.15\textwidth}
    \includegraphics[width=\textwidth]{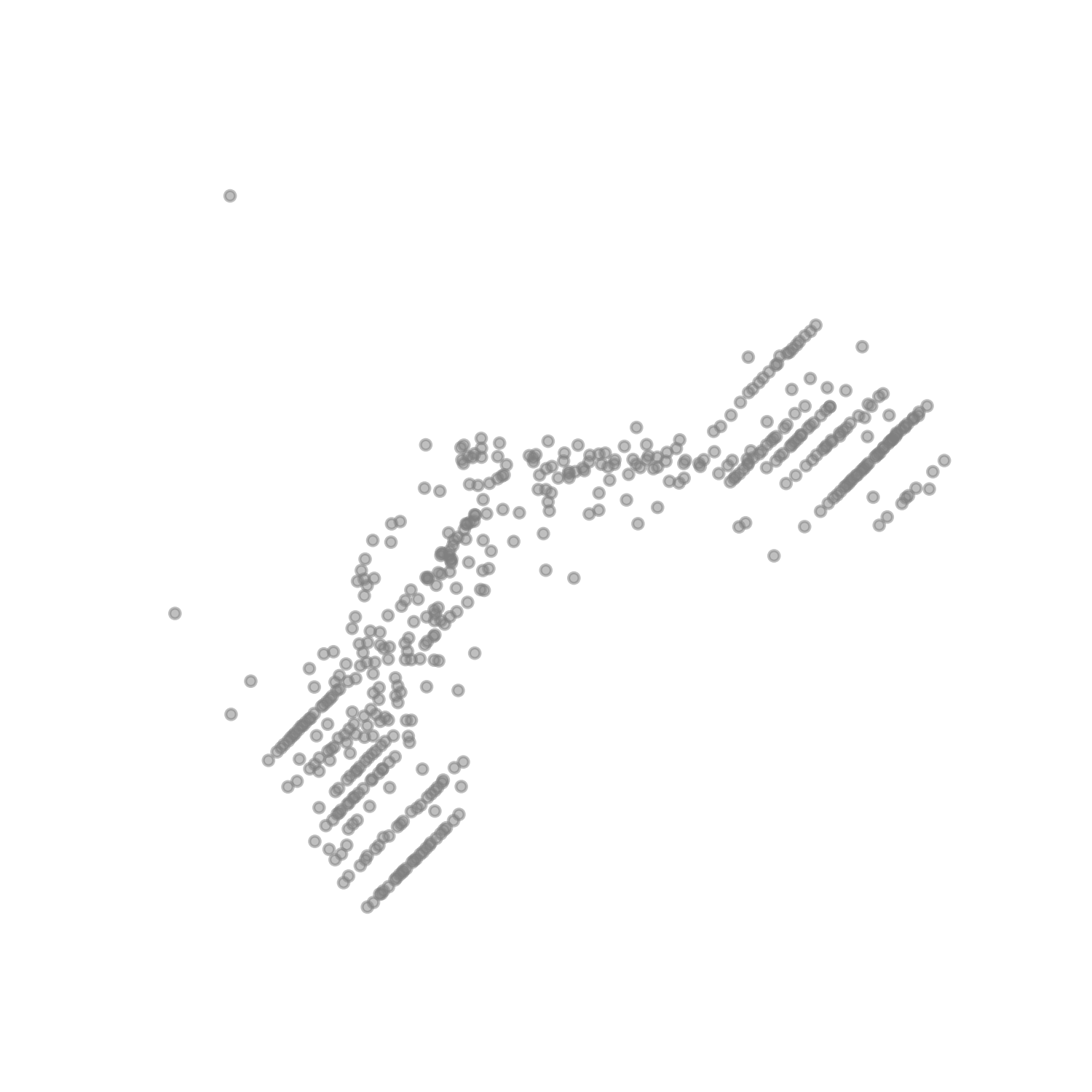}
    \caption*{AE}
  \end{subfigure}
  \hfill
  \begin{subfigure}[b]{.15\textwidth}
    \includegraphics[width=\textwidth]{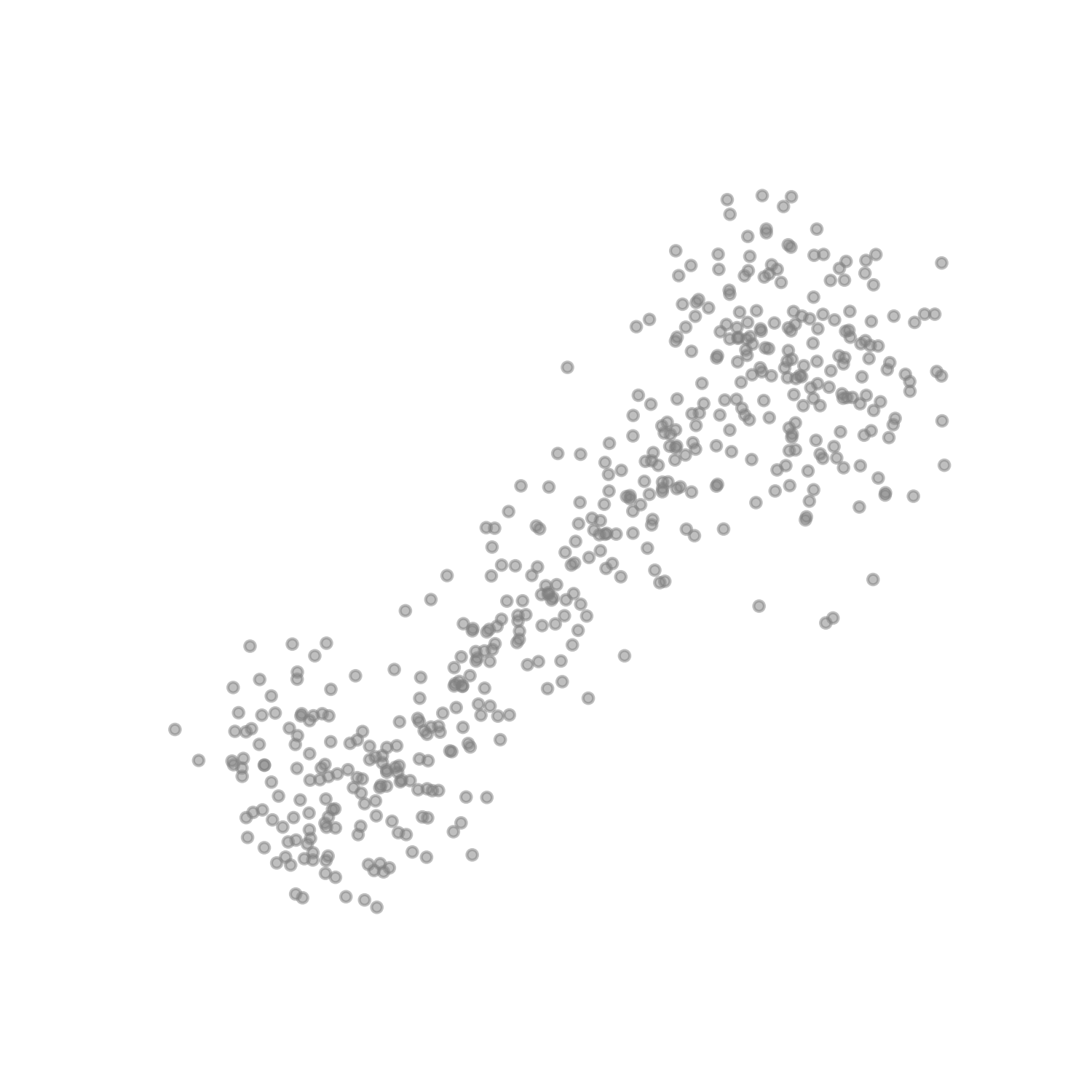}
    \caption*{STLB}
  \end{subfigure}
  \end{subfigure}
  \hfill
    \begin{subfigure}[t]{\linewidth}
  \begin{subfigure}[b]{.4\textwidth}
    \includegraphics[width=\textwidth]{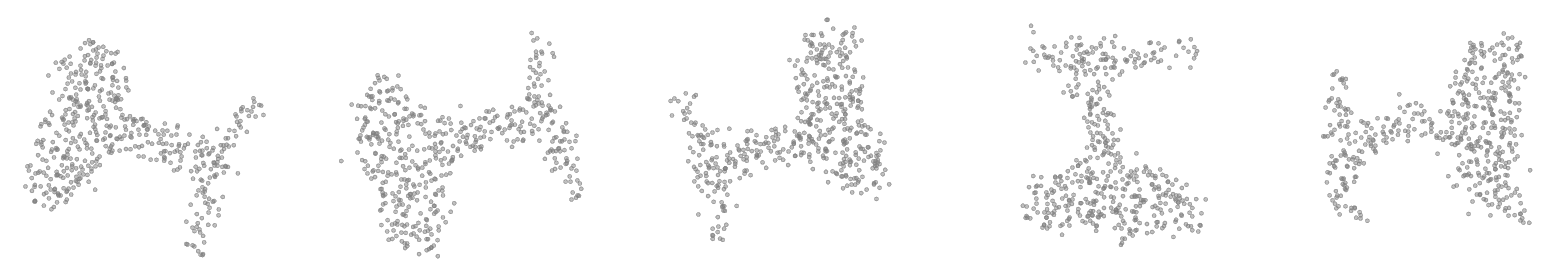}
    \caption*{Goblet Inputs}
  \end{subfigure}
  \hfill
  \begin{subfigure}[b]{.15\textwidth}
    \includegraphics[width=\textwidth]{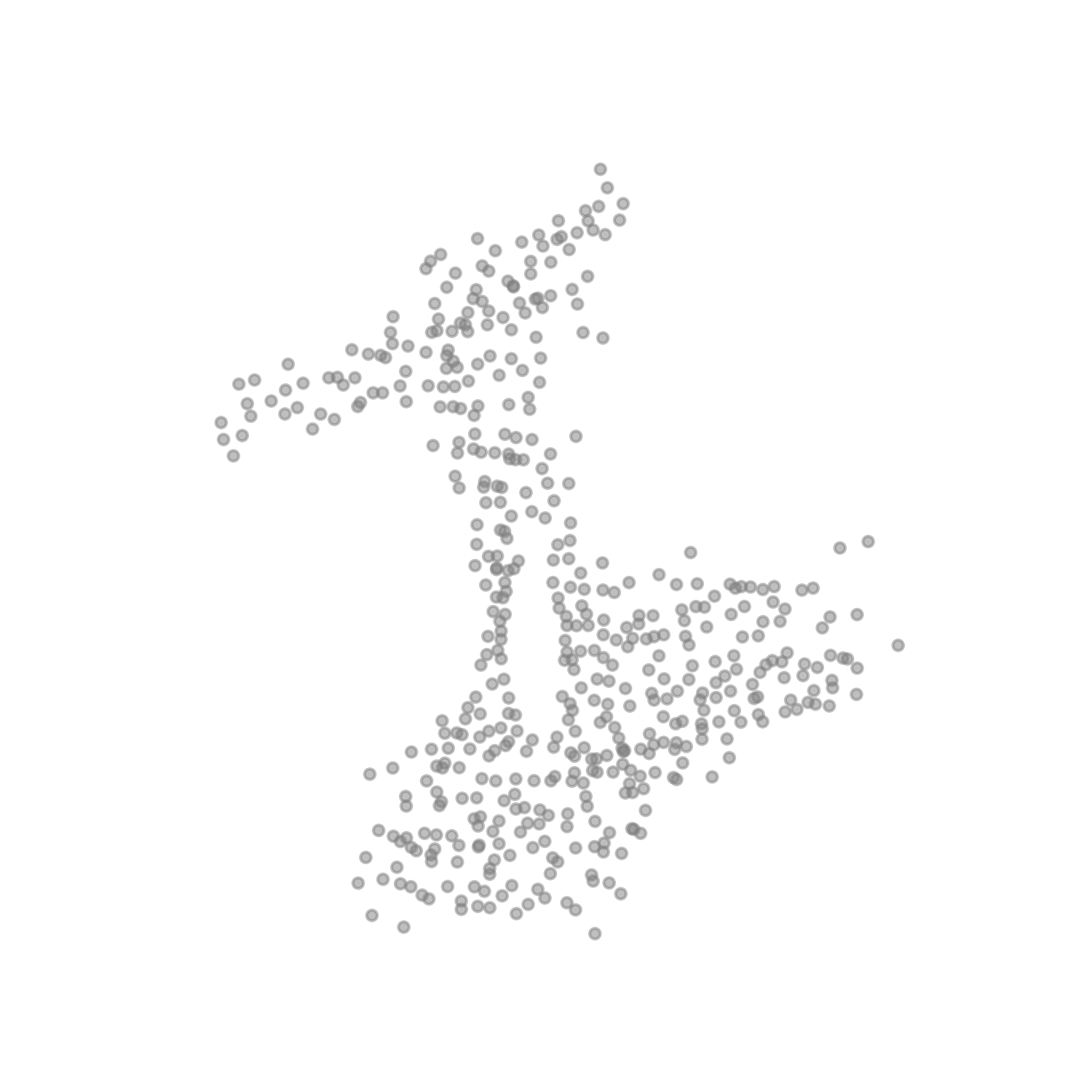}
    \caption*{TLB}
  \end{subfigure}
  \hfill
  \begin{subfigure}[b]{.15\textwidth}
    \includegraphics[width=\textwidth]{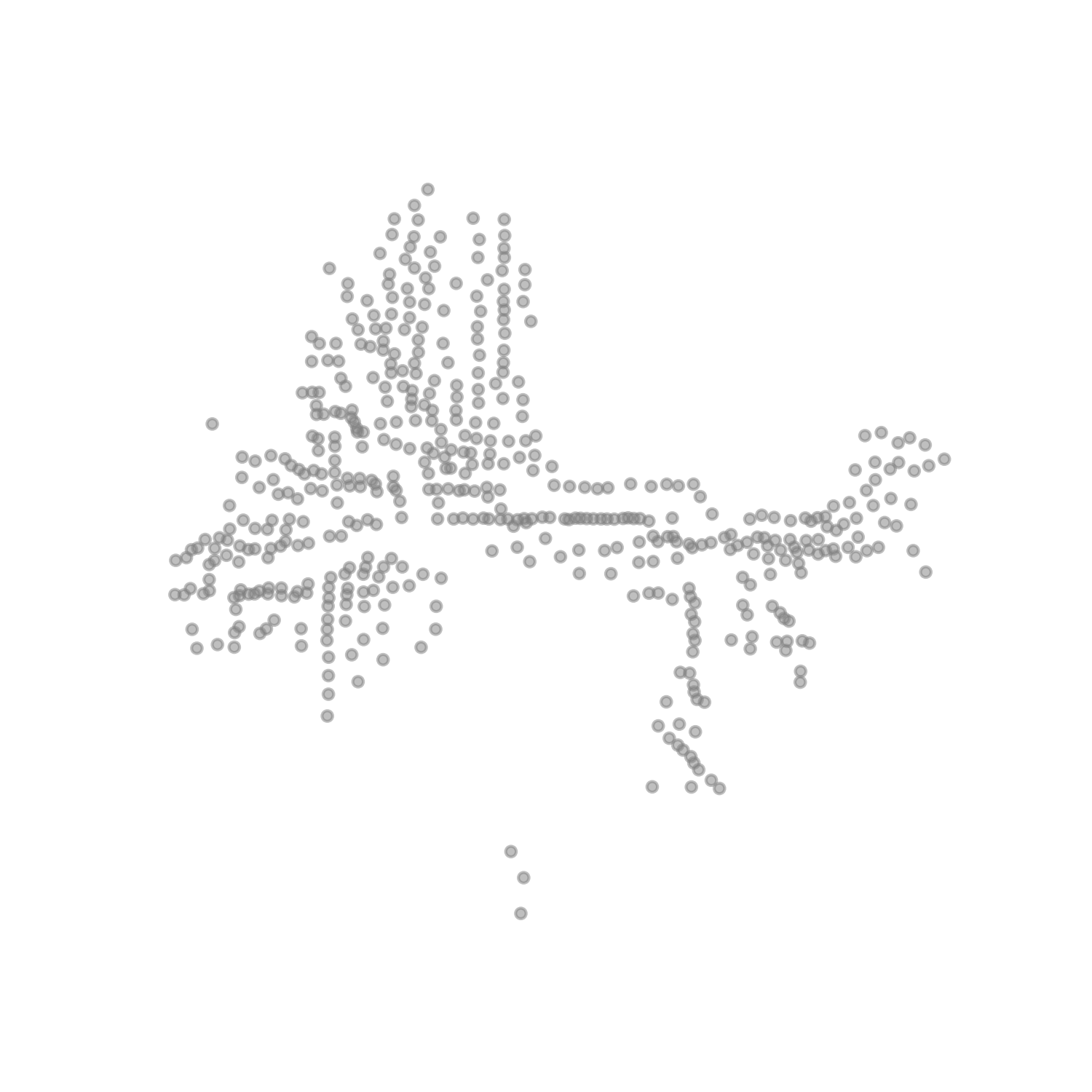}
    \caption*{AE}
  \end{subfigure}
  \hfill
  \begin{subfigure}[b]{.15\textwidth}
    \includegraphics[width=\textwidth]{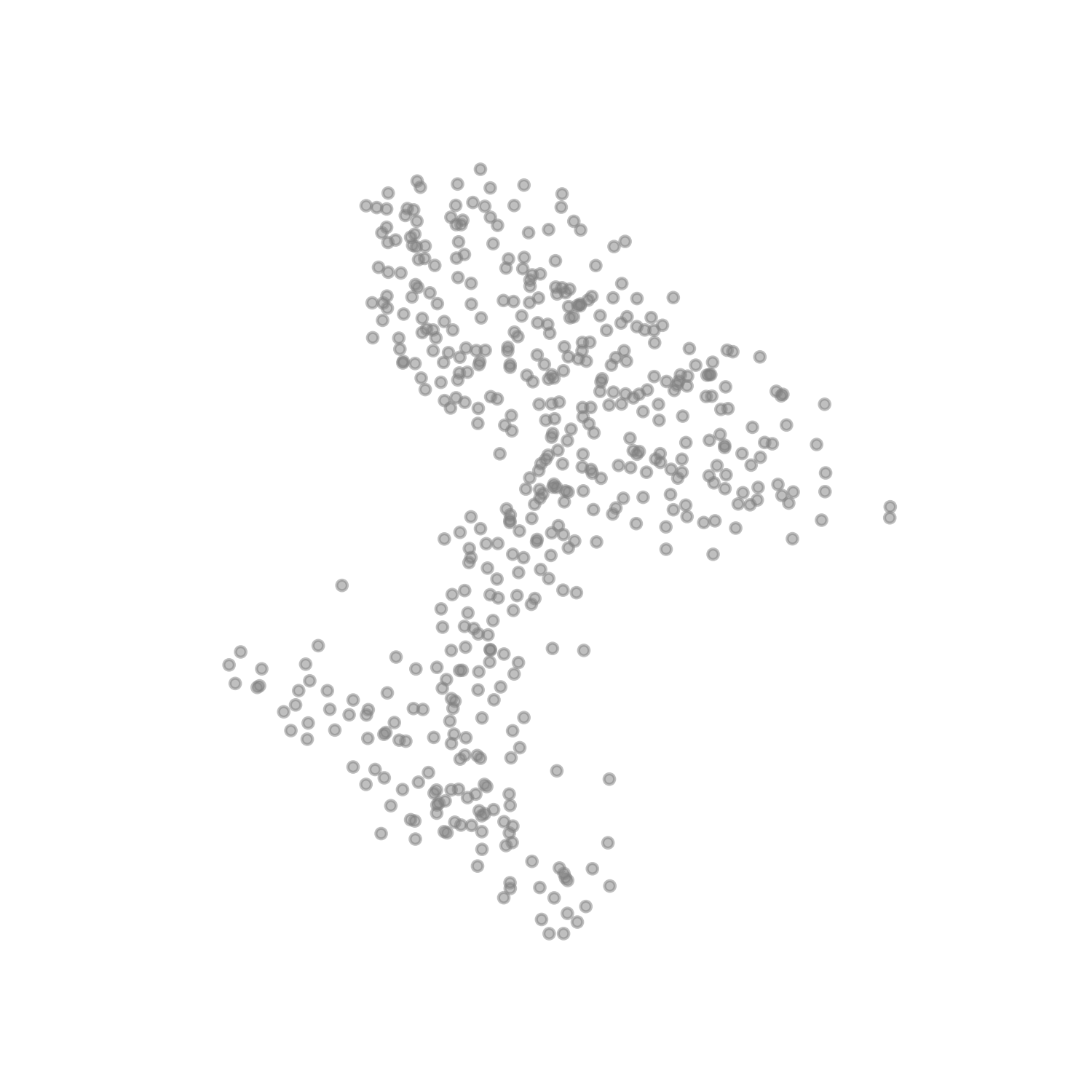}
    \caption*{STLB}
  \end{subfigure}
  \end{subfigure}
  \hfill
    \begin{subfigure}[t]{\linewidth}
  \begin{subfigure}[b]{.4\textwidth}
    \includegraphics[width=\textwidth]{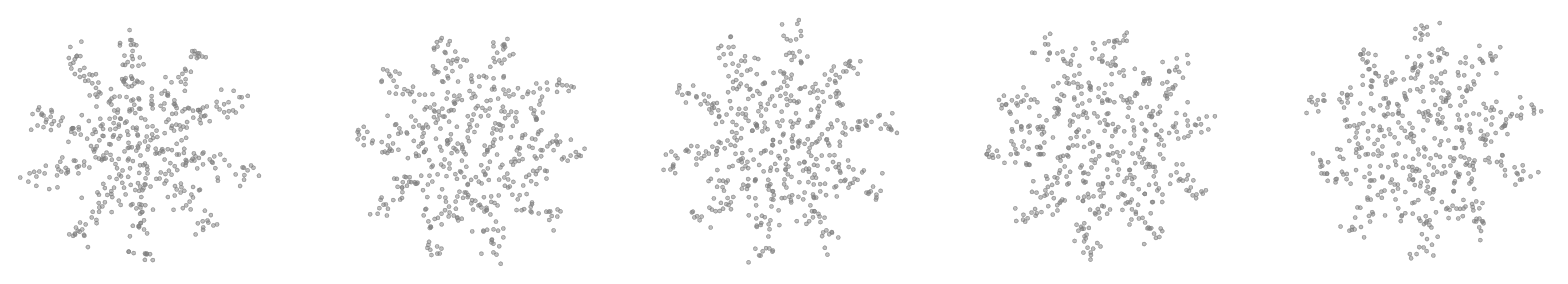}
    \caption*{Star Inputs}
  \end{subfigure}
  \hfill
  \begin{subfigure}[b]{.15\textwidth}
    \includegraphics[width=\textwidth]{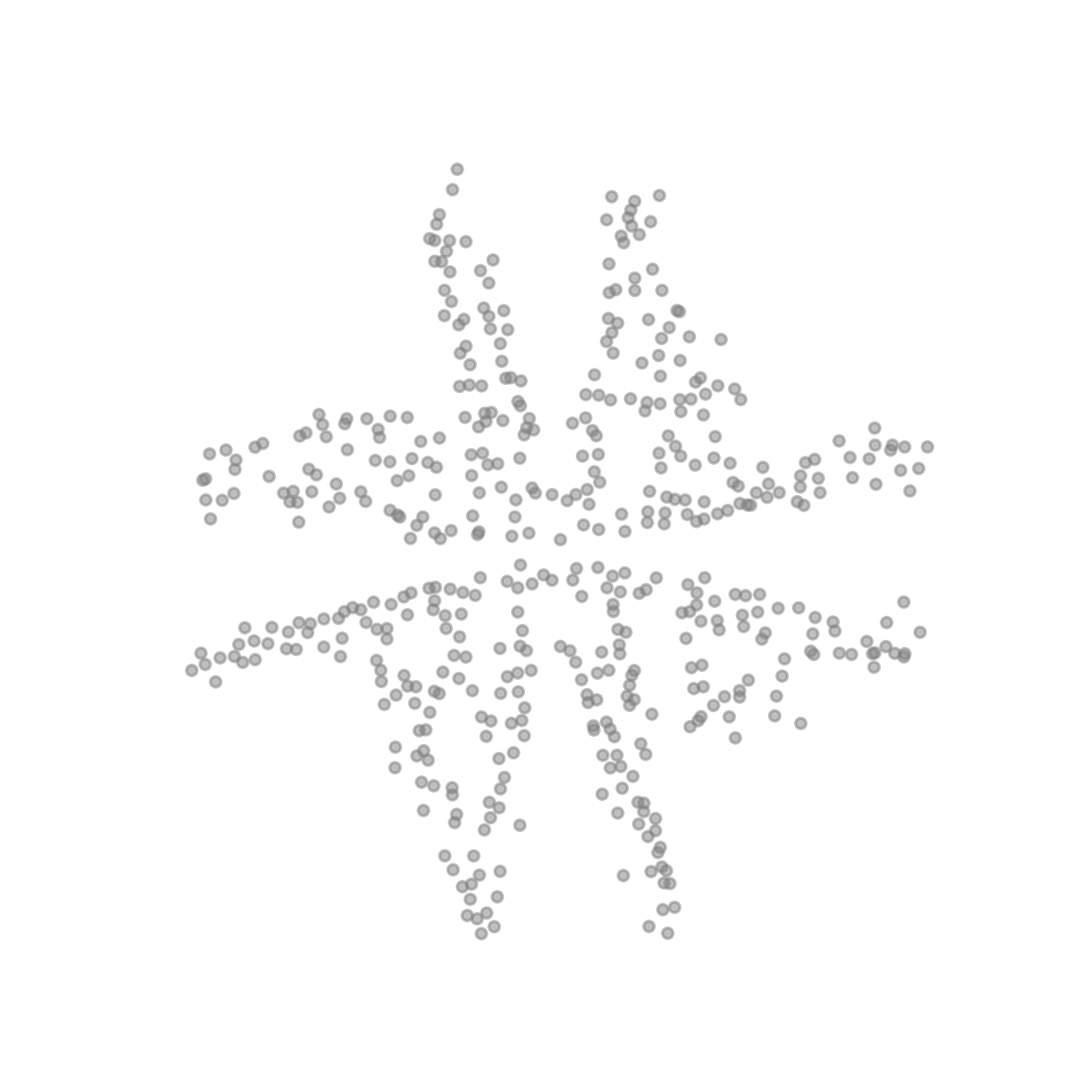}
    \caption*{TLB}
  \end{subfigure}
  \hfill
  \begin{subfigure}[b]{.15\textwidth}
    \includegraphics[width=\textwidth]{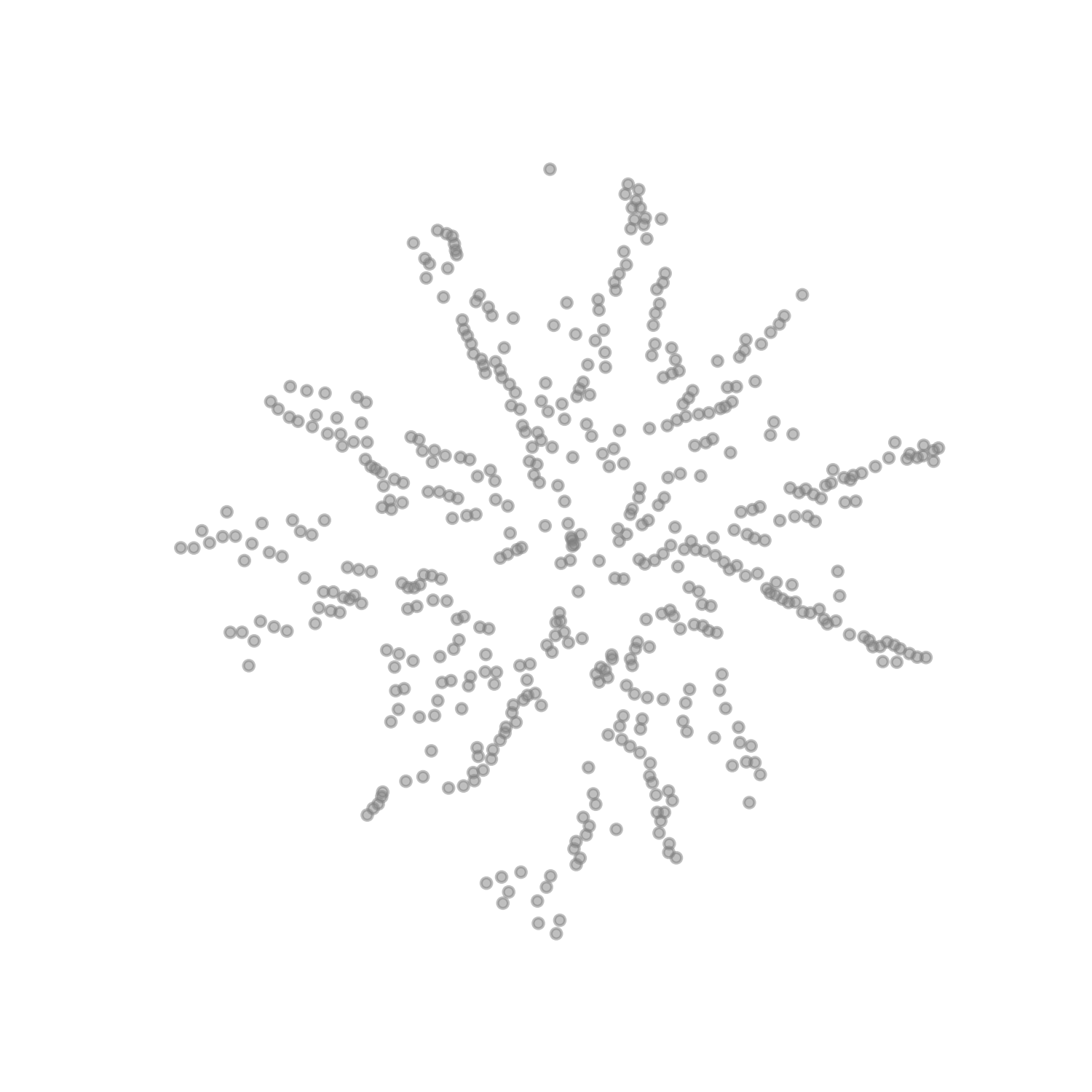}
    \caption*{AE}
  \end{subfigure}
  \hfill
  \begin{subfigure}[b]{.15\textwidth}
    \includegraphics[width=\textwidth]{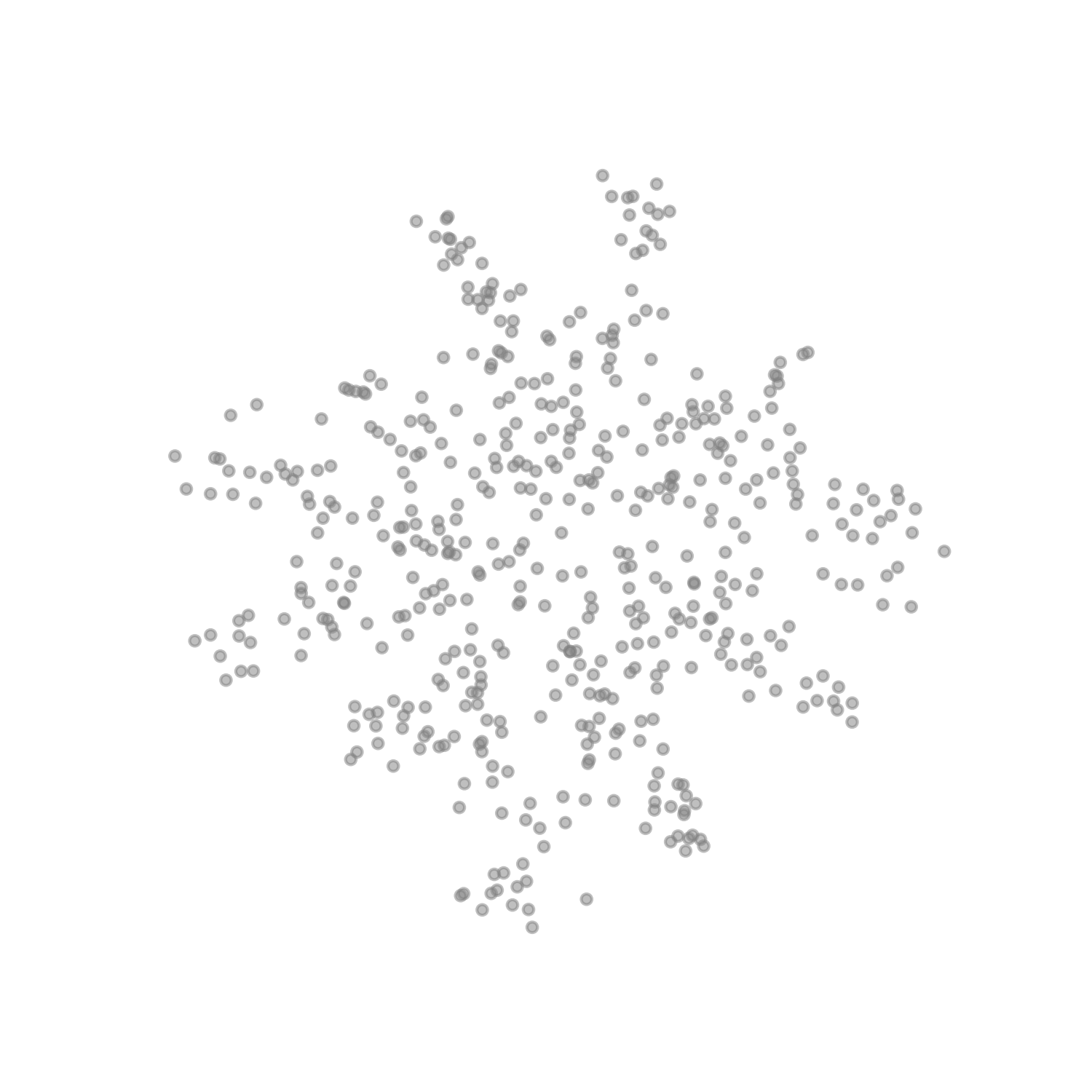}
    \caption*{STLB}
  \end{subfigure}
  \end{subfigure}
  \hfill
  %   \begin{subfigure}[t]{\linewidth}
  % \begin{subfigure}[b]{.4\textwidth}
  %   \includegraphics[width=\textwidth]{figures/2dShapes_Bary/horseshoe_inputs.png}
  %   \caption*{Horseshoe Inputs}
  % \end{subfigure}
  % \hfill
  % \begin{subfigure}[b]{.15\textwidth}
  %   \includegraphics[width=\textwidth]{figures/2dShapes_Bary/horseshoe_bary_TLB.png}
  %   \caption*{TLB}
  % \end{subfigure}
  % \hfill
  % \begin{subfigure}[b]{.15\textwidth}
  %   \includegraphics[width=\textwidth]{figures/2dShapes_Bary/horseshoe_bary_Energy.png}
  %   \caption*{AE}
  % \end{subfigure}
  % \hfill
  % \begin{subfigure}[b]{.15\textwidth}
  %   \includegraphics[width=\textwidth]{figures/2dShapes_Bary/horseshoe_bary_STLB.png}
  %   \caption*{STLB}
  % \end{subfigure}
  %  \end{subfigure}
  \caption{Euclidean 2D TLB\revise{, AE} and STLB barycenters for five samples of different shape classes.}
  \label{fig:2dbary}
\end{figure}

GW barycenters are generalized Fréchet means 
that interpolate several mm-spaces $\YY_k$, $k=1,\dots,K$
by solving
\begin{equation}
    \argmin
    \Bigl\{
    \sum_{k=1}^K \GW_2^2(\XX, \YY_k)
    \Bigm\vert
    \XX \; \text{mm-space}
    \Bigr\}.
    \label{eq:GW-bary}
\end{equation}
The solutions of the GW barycenter problem can be completely characterized
using multi-marginal OT
\cite{beier2023multi,floriantangential};
its numerical computation, however, remains challenging.
Established free-support methods \cite{PeyreCuturi2019} 
minimize over the distance of $\XX$
producing non-Euclidean outputs;
whereas 
fixed-support methods \cite{beier2023multi, beier2025joint}
limit the resolution of the barycenter.
As an alternative, 
we replace GW in \eqref{eq:GW-bary} by TLB, STLB ($r=500, L=500)$ and the so-called Anchor Energy (AE) distance \cite{sato2020fast} for comparison.
For each distance, we perform a gradient descent over the points $x_i \in \R^d$ in
$\XX = (\{x_i\}_{i=1}^{n}, d_{\mathrm{E}}, \iota\}$, where
$d_{\mathrm E}$ is the Euclidean metric
and $\iota$ the uniform measure.
Figure \ref{fig:2dbary} shows computed 2D barycenters
for different upsampled shapes ($n=500$) from \cite{beier2022linear},
obtained via three identical random restart initializations, and 1000 steps (width: 0.1).

\subsection{Classification of 2D and 3D Shapes}
\label{subsec:2shapes}

\sisetup{
  round-mode      = places,
  round-precision = 1
}
\begin{table*}[ht]
  \centering
  \fontsize{8}{7.5}\selectfont
  \setlength{\tabcolsep}{4pt}  % tighter columns
  \renewcommand{\arraystretch}{1.1}
  \begin{tabular}{ll
                  c c
                  c c
                  c c
                  c c}
    \toprule
    distance & type
      & \multicolumn{2}{c}{2D shapes}
      & \multicolumn{2}{c}{animals}
      & \multicolumn{2}{c}{FAUST-500}
      & \multicolumn{2}{c}{FAUST-1000} \\
    \cmidrule(lr){3-4} \cmidrule(lr){5-6}
    \cmidrule(lr){7-8} \cmidrule(lr){9-10}
      & 
      & acc. (\%) & time (ms)
      & acc. (\%) & time (ms)
      & acc. (\%) & time (ms)
      & acc. (\%) & time (ms) 
      \\
    \midrule
    SLB        & global
      & 97.9$\pm$2.2  & 0.6
      & 100.0$\pm$0.1 & 0.7
      & 31.2$\pm$5.4  & 48.3
      & 35.2 $\pm$ 5.7          & 183.7       \\
      \hline
    TLB        & local
      & 100.0$\pm$0.3 & 0.7
      & 100.0$\pm$0.0 & 0.8
      & 36.8$\pm$5.4  & 24.8
      & 40.0 $\pm$ 6.0          & 88.3       \\
    STLB & local
      & 99.3$\pm$1.2  & 1.3
      & 99.2$\pm$1.1  & 1.5
      & 35.8$\pm$5.5  & 11.6
      & 40.1$\pm$5.7  & 23.1       \\
    AE & local
      & 99.7$\pm$0.9  & 0.6
      & 97.8$\pm$1.8  & 0.7
      & 37.8$\pm$5.6  & 8.1
      & 41.6$\pm$5.7          & 35.8       \\
      \hline
    GW         & iso
      & 99.7$\pm$0.6  & 2.6
      & 100.0$\pm$0.0 & 5.8
      & 29.2$\pm$4.3  & 482.0
      & 33.3$\pm$5.2          & 2398.3       \\
    \hline
    \revise{LGW}         & approx
      & 91.2$\pm$2.7  & 288.9
      & 89.8$\pm$3.9 & 481.0
      & 29.4$\pm$4.7  & $ 7 \cdot 10^4$
      & 30.2$\pm$5.1          & $  10^5$       \\
    \revise{qGW}         & approx
      & 61.6$\pm$4.4  & 16.4
      & 48.8$\pm$6.3 & 14.7
      & 13.7$\pm$3.6  & 18.2
      & 12.4$\pm$3.2  & 35.3       \\
    \bottomrule
  \end{tabular}
    \caption{KNN Shape classification: 
    Mean accuracy (acc.) with standard deviation
    and mean runtime (time).}
  \label{tab:knn_shapes}
\end{table*}
\begin{table*}[t]
  \centering
  \fontsize{8}{7.5}\selectfont
  \setlength{\tabcolsep}{4pt}
  \renewcommand{\arraystretch}{1.1}
  \begin{tabular}{l
    *{6}{c}
  }
    \toprule
    graph model
      & \multicolumn{1}{c}{WS}
      & \multicolumn{1}{c}{BA}
      & \multicolumn{1}{c}{RR}
      & \multicolumn{1}{c}{WS}
      & \multicolumn{1}{c}{BA}
      & \multicolumn{1}{c}{RR} \\
      nodes
      & 10 & 10 & 10 & 500 & 500 & 100
      \\
      parameters
    & \multicolumn{1}{c}{{$k{=}4$, $p_E{=}0.1$}}
      & \multicolumn{1}{c}{{$m{=}5$, $p_B{=}0.5$}}
      & \multicolumn{1}{c}{{$R=3$, $p_B=0.5$}}
    & \multicolumn{1}{c}{{$k{=}250$, $p_E{=}0.1$}}
      & \multicolumn{1}{c}{{$m{=}250$, $p_B{=}0.1$}}
      & \multicolumn{1}{c}{{$R=3$, $p_B=0.5$}}\\
    \midrule
    SLB   
      & 90.40 $\pm$ 2.62
      & 61.50 $\pm$ 3.44
      & 96.30 $\pm$ 1.62
      & 50.00 $\pm$ 0.00
      & 50.00 $\pm$ 0.00
      & 100.00 $\pm$ 0.00 \\
    TLB   
      & 96.80 $\pm$ 2.99
      & 99.50 $\pm$ 1.02
      & 99.40 $\pm$ 1.80
      & 100.00 $\pm$ 0.00
      & 100.00 $\pm$ 0.00
      & 100.00 $\pm$ 0.00 \\
    \textbf{STLB}
      & 99.00 $\pm$ 1.00
      & 99.80 $\pm$ 0.40
      & 99.80 $\pm$ 0.60
      & 99.60 $\pm$ 0.80
      & 99.40 $\pm$ 0.92
      & 100.00 $\pm$ 0.00 \\
    GW    
      & 67.40 $\pm$ 3.83
      & 73.30 $\pm$ 3.85
      & 65.50 $\pm$ 4.03
      & 65.00 $\pm$ 5.00
      & 100.00 $\pm$ 0.00
      & 100.00 $\pm$ 0.00 \\
    WL-D   
      & 100.00 $\pm$ 0.00
      & 100.00 $\pm$ 0.00
      & 50.00 $\pm$ 0.00
      & 100.00 $\pm$ 0.00
      & 100.00 $\pm$ 0.00
      & 50.00 $\pm$ 0.00 \\
    \hline
    \textbf{FTLB}
      & 100.00 $\pm$ 0.00
      & 96.00 $\pm$ 1.18
      & 97.60 $\pm$ 1.50
      & 100.00 $\pm$ 0.00
      & 94.80 $\pm$ 1.60
      & 98.60 $\pm$ 0.92 \\
    \textbf{SFTLB}
      & 100.00 $\pm$ 0.00
      & 98.20 $\pm$ 1.08
      & 99.40 $\pm$ 0.92
      & 100.00 $\pm$ 0.00
      & 100.00 $\pm$ 0.00
      & 99.80 $\pm$ 0.60 \\
    FGW   
      & 98.20 $\pm$ 1.66
      & 90.30 $\pm$ 1.79
      & 74.60 $\pm$ 2.69
      & 100.00 $\pm$ 0.00
      & 98.80 $\pm$ 1.33
      & 69.80 $\pm$ 4.42 \\
    WL-F   
      & 100.00 $\pm$ 0.00
      & 83.10 $\pm$ 3.21
      & 81.30 $\pm$ 3.85
      & 100.00 $\pm$ 0.00
      & 99.20 $\pm$ 1.33
      & 86.80 $\pm$ 3.92 \\
    \bottomrule
  \end{tabular}
  \caption{Mean accuracy (\%) with standard deviation for graph isomorphism testing. 
  The first five methods are label-agnostic, whereas the last four methods employ node features.
  Proposed methods are printed in bold.}
  \label{tab:iso_test_eval}
\end{table*}

A key application of the GW distance is the comparison of shapes,
which can be represented as Euclidean point clouds 
or as meshes equipped with geodesic distances.
Building on experiments from \cite{memoli2011gromov,beier2022linear}, 
we rely on the k-nearest neighbor (KNN) classification \cite{cover1967nearest}. 
More precisely,
we first precompute pairwise distance matrices 
with respect to SLB, TLB, STLB \revise{($r=10, L=100$)}, AE \cite{sato2020fast}, \revise{qGW \cite{chowdhury2019gromov}, LGW \cite{beier2022linear}
} and GW.
Second, we evaluate the classification accuracy 
by assigning each test point to the majority class 
among its three nearest neighbors. 
We report the mean accuracy 
over 1000 random 25\%/75\% training/test splits 
on the following datasets:
\begin{description}
    \item[2D Shapes:] 80 shapes from 4 classes (bone, goblet, star, horseshoe)
        represented as Euclidean point clouds, each consisting of 50 points \cite{beier2022linear}; 
    \item[Animals:] 73 human/animal shapes from 7 classes \cite{sumner2004deftransferdata_animals}
        downsampled to 3D meshes with 50 vertices, equipped with the geodesic distance
        \cite{memoli2011gromov,beier2022linear};
    \item[FAUST-500 and FAUST-1000:] 100 human meshes from 10 subjects \cite{Bogo:CVPR:2014:FAUST},
        downsampled and represented by 3D meshes with 500 and 1000 vertices
        equipped with the geodesic distance
        using PyVista \cite{sullivan2019pyvista} and NetworkX \cite{hagberg2008exploring}.
\end{description}
Table~\ref{tab:knn_shapes} summarizes classification accuracy, runtime, and nullspace types \cite{memoli2022distance}:
SLB vanishes for globally matching distance distributions (`global'), 
TLB/STLB/AE for locally matching distributions (`local'),
and GW for isomorphic shapes (`iso'). \revise{Lastly, LGW and qGW approximate GW (`approx')\footnote{\revise{LGW/qGW depend heavily on hyperparameter choices balancing speed and performance, see our appendix for a discussion.}}.}
While most methods perform near-perfectly on `2D shapes' and `animals', 
accuracy drops on FAUST datasets. 
On these, 
`local' methods (TLB, STLB, AE) consistently outperform \revise{other approaches}.
Although GW should yield perfect classifications in theory, 
the underlying non-convex quadratic minimization reduces the numerical accuracy.
STLB shows a clear runtime advantage for large-scale datasets.

\subsection{Detection of Graph Isomorphisms}

An important theoretical property of the GW and FGW distance 
is that they vanish only for two isomorphic inputs 
\cite{memoli2011gromov, vayer2020fused},
which facilitates an application in graph isomorphism testing. 
In practice,
the established GW and FGW solvers are however not accurate enough
to give reliable results of this NP-hard problem
\cite{grohe2020graph}. 
Therefore,
SLB has been proposed as a statistical test for graph isomorphisms 
\cite{weitkamp2022gromov, weitkamp2024distribution}.
However,
SLB has less statistical power than TLB \cite{memoli2022distance};
for this reason,
we propose to use STLB and SFTLB instead.
During the experiment,
we rely on synthetic datasets of random graphs 
with fixed node numbers,
where half of the graph pairs are isomorphic 
and half of the graph pairs are not.
Here, we consider these graph generation models:
\begin{description}
    \item[Watts--Strogatz (WS) Graph:]
        The WS graph is governed
        by a neighborhood number $k$ 
        and an edge rewiring probability $p_E$
        \cite{watts1998collective}.
        We add node features according to a 1D standard normal distribution.
    \item[Barabasi--Albert (BA) Graph:]
        The BA graph is governed by an edge generation number $m$ \cite{barabasi1999emergence}.
        We assign node features in $\{0,1\}$ based on a Bernoulli distribution 
        with probability $p_B$.
    \item [Random Regular (RR) Graph:] 
        The RR graph is based on a regularity $R$ \cite{bollobas1998random}.
        We assign Bernoulli node features in $\{0,1\}$
        according to the probability $p_B$.
\end{description}
We compare the generated, unlabeled graph pairs using
SLB, TLB, STLB, GW, and Weisfeiler-Lehman (WL) kernel test \cite{shervashidze2011wlkernel}.
For the WL test, 
we use 5 iterations and set the labels to the node degree (WL-D).
The labeled graph pairs are compared using $\alpha=0.5$ for
FTLB, 
SFTLB ($r=5$, $L=100$), 
and entropic FGW ($\varepsilon=10^{-3}$), 
as well as with WL
using 5 iterations 
and (binned) node features (WL-F).
We classify two graphs as isomorphic if they are among the 50\% most similar graph pairs in terms of the different distances or if the distance becomes exactly zero.
In Table~\ref{tab:iso_test_eval}, we display the results of 10 experiment repetitions over 6 different random graph configurations.

\section{Conclusion}
% In this work, 
% we address the computational challenges of GW and FGW
% by developing novel lower bounds. 
% For this,
% we extend TLB to the fused case
% and reveal its connection 
% to a Wasserstein problem. 
% Leveraging this connection
% and employing a numerical quadrature scheme 
% for the underlying 1D OTs,
% we formulate an effective sliced FGW.
% This novel distance defines a pseudo-metric 
% interpolating between TLB and SW.
% We also shed light on the dimensional dependence of SW.
% Our experiments underscore efficiency 
% and applicability of our FGW variants 
% for shape comparison and graph isomorphism testing.
\revise{We address the computational limits of GW and FGW by introducing new lower bounds. Extending TLB to the fused setting and linking it to a Wasserstein problem, we derive an efficient sliced FGW via numerical quadrature of 1D OTs. Our work bridges TLB and SW, clarifies SW’s dimensional behavior, and proves effective for shape and graph comparisons, thereby overcoming the limitations of sliced GW \cite{vayer2019sgw}.}
% Besides contributing to
% the literature on sliced hierarchical OT \cite{piening2025smw, nguyen2025sotdd}
% we overcome the limitations of sliced GW \cite{vayer2019sgw}.
In future research, the calculation of  local distance distribution required for AE, FTLB, and SFTLB may be further accelerated via sweep-line algorithms \cite{sato2020fast}, which is not studied in this work.

\section*{Acknowledgments}
M.P. acknowledges funding from the German Research Foundation (DFG) within the project BIOQIC (GRK2260/289347353). 
\bibliography{references}

@article{memoli2011gromov,
  title={{G}romov--{W}asserstein distances and the metric approach to object matching},
  author={M{\'e}moli, Facundo},
  journal={Foundations of Computational Mathematics},
  volume={11},
  pages={417--487},
  year={2011},
  publisher={Springer}
}

@article{memoli2022distance,
  title={Distance distributions and inverse problems for metric measure spaces},
  author={M{\'e}moli, Facundo and Needham, Tom},
  journal={Studies in Applied Mathematics},
  volume={149},
  number={4},
  pages={943--1001},
  year={2022},
  publisher={Wiley Online Library}
}

@article{chowdhury2019gromov,
  author  = {Chowdhury, Samir and M\'emoli, Facundo},
  title   = {The {{G}romov--W}asserstein distance between networks and stable network invariants},
  journal = {Information and Inference: A Journal of the IMA},
  volume  = {8},
  number  = {4},
  pages   = {757--787},
  year    = {2019},
  month   = nov,
  doi     = {10.1093/imaiai/iaz026},
}

@phdthesis{weitkamp2022gromov,
  author       = {Weitkamp, Christoph Alexander},
  title        = {{{G}romov--W}asserstein Distances and their Lower Bounds},
  school       = {Georg-August-Universit{\"a}t G{\"o}ttingen},
  year         = {2022},
  month        = nov,
  type         = {Doctoral thesis},
  url          = {http://dx.doi.org/10.53846/goediss-9587},
  advisor      = {Axel Munk},
}

@inproceedings{jin2022orthogonal,
  author       = {Jin, Hongwei and Yu, Zishun and Zhang, Xinhua},
  title        = {Orthogonal {{G}romov--W}asserstein Discrepancy with Efficient Lower Bound},
  booktitle    = {Procedings of UAI'22},
  series       = {Proceedings of Machine Learning Research},
  volume       = {180},
  pages        = {917--927},
  year         = {2022},
  publisher    = {PMLR},
  url          = {https://proceedings.mlr.press/v180/jin22a.html},
}

@book{Santambrogio2015,
  author       = {Santambrogio, Filippo},
  title        = {Optimal Transport for Applied Mathematicians: Calculus of Variations, PDEs and Modeling},
  hideseries       = {Progress in Nonlinear Differential Equations and Their Applications},
  hidevolume       = {87},
  publisher    = {Birkh{\"a}user},
  address = {Cham},
  year         = {2015},
  hideisbn         = {978-3-319-20828-2},
  doi          = {10.1007/978-3-319-20828-2}
}

@article{PeyreCuturi2019,
  author       = {Peyr{\'e}, Gabriel and Cuturi, Marco},
  title        = {Computational optimal transport: with applications to data science},
  journal      = {Foundations and Trends{\textregistered} in Machine Learning},
  volume       = {11},
  number       = {5-6},
  pages        = {355--607},
  year         = {2019},
  doi          = {10.1561/2200000073},
  publisher    = {Now Publishers, Inc.}
}

@article{vayer2020fused,
  title = {Fused {G}romov--{W}asserstein distance for structured objects},
  author = {Vayer, Titouan and Chapel, Laetitia and Flamary, R{\'e}mi and Tavenard, Romain and Courty, Nicolas},
  journal = {Algorithms},
  volume = {13},
  number = {9},
  pages = {212},
  year = {2020},
  publisher = {MDPI},
  doi = {10.3390/a13090212}
}

@book{villani2003topics,
  title     = {Topics in Optimal Transportation},
  author    = {Villani, Cédric},
  year      = {2003},
  publisher = {American Mathematical Society},
  hideseries    = {Graduate Studies in Mathematics},
  hidevolume    = {58},
  address   = {Providence, RI},
  hideisbn      = {978-0-8218-3312-4},
  doi       = {10.1090/gsm/058},
  url       = {https://bookstore.ams.org/gsm-58}
}

@inproceedings{nadjahi2020statistical,
  title = {Statistical and topological properties of sliced probability divergences},
  author = {Nadjahi, Kimia and Durmus, Alain and Chizat, L{\'e}na{\"i}c and Kolouri, Soheil and Shahrampour, Shahin and Simsekli, Umut},
  booktitle = {Advances in Neural Information Processing Systems},
  volume = {33},
  pages = {15657--15669},
  year = {2020},
  publisher = {Curran Associates},
  url = {https://proceedings.neurips.cc/paper/2020/file/eefc9e10ebdc4a2333b42b2dbb8f27b6-Paper.pdf}
}

@ARTICLE{KPR16,
  author={Kolouri, Soheil and Park, Se Rim and Rohde, Gustavo K.},
  journal={IEEE Transactions on Image Processing}, 
  title={The {Radon} cumulative distribution transform and its application to image classification}, 
  year={2016},
  volume={25},
  number={2},
  pages={920--934},
  doi={10.1109/TIP.2015.2509419}}

@article{bonneel2015slicedbarycenters,
  title={Sliced and {Radon W}asserstein barycenters of measures},
  author={Bonneel, Nicolas and Rabin, Julien and Peyr{\'e}, Gabriel and Pfister, Hanspeter},
  journal={Journal of Mathematical Imaging and Vision},
  volume={51},
  pages={22--45},
  year={2015},
  publisher={Springer},
    doi={10.1007/s10851-014-0506-3}
}

@article{beier2022linear,
  title={On a linear {G}romov--{W}asserstein distance},
  author={Beier, Florian and Beinert, Robert and Steidl, Gabriele},
  journal={IEEE Transactions on Image Processing},
  volume={31},
  pages={7292--7305},
  year={2022},
  publisher={IEEE},
  doi={10.1109/TIP.2022.3216475}
}

@article{barabasi1999emergence,
  author    = {Barabási, Albert-László and Albert, Réka},
  title     = {Emergence of scaling in random networks},
  journal   = {Science},
  volume    = {286},
  number    = {5439},
  pages     = {509--512},
  year      = {1999},
  doi       = {10.1126/science.286.5439.509}
}

@misc{sumner2004deftransferdata_animals,
  author       = {Robert W. Sumner and Jovan Popović},
  title        = {Mesh data from deformation transfer for triangle meshes},
  year         = {2004},
  url          = {http://people.csail.mit.edu/sumner/research/deftransfer/data.html},
  note         = {Accessed: 2025-06-03},
  howpublished = {Available online at \url{http://people.csail.mit.edu/sumner/research/deftransfer/data.html}},
  publisher    = {MIT Computer Science and Artificial Intelligence Laboratory (CSAIL)},
  institution  = {Massachusetts Institute of Technology}
}

@inproceedings{Bogo:CVPR:2014:FAUST,
  title     = {{FAUST}: Dataset and evaluation for {3D} mesh registration},
  author    = {Federica Bogo and Javier Romero and Matthew Loper and Michael J. Black},
  booktitle = {Proceedings of CVPR'14},
  year      = {2014},
  publisher = {IEEE}
}

@article{floriantangential,
  title={Tangential fixpoint iterations for {Gromov--W}asserstein barycenters},
  author={Beier, Florian and Beinert, Robert},
  journal={SIAM Journal on Imaging Sciences},
  volume={18},
  number={2},
  pages={1058--1100},
  year={2025}
}

@article{sullivan2019pyvista,
  doi         = {10.21105/joss.01450},
  url         = {https://doi.org/10.21105/joss.01450},
  year        = {2019},
  month       = {May},
  publisher   = {The Open Journal},
  volume      = {4},
  number      = {37},
  pages       = {1450},
  author      = {Bane Sullivan and Alexander Kaszynski},
  title       = {{PyVista}: {3D} plotting and mesh analysis through a streamlined interface for the {Visualization Toolkit} ({VTK})},
  journal     = {Journal of Open Source Software}
}

@incollection{hagberg2008exploring,
  title     = {Exploring network structure, dynamics, and function using {NetworkX}},
  author    = {Hagberg, Aric A. and Schult, Daniel A. and Swart, Pieter J.},
  booktitle   = {Proceedings of SciPy'08},
  pages     = {11--15},
  year      = {2008},
  editor    = {Varoquaux, Gaël and Vaught, Travis and Millman, Jarrod},
  publisher = {Curvenote}
}

@inproceedings{beier2025joint,
title={Joint metric space embedding by unbalanced pptimal transport with {G}romov--{W}asserstein marginal penalization},
author={Florian Beier and Moritz Piening and Robert Beinert and Gabriele Steidl},
booktitle={Proceedings of ICML'25},
year={2025},
url={https://openreview.net/forum?id=0YZHfUmsJv},
publisher={OpenReview.net}
}

@article{beier2023multi,
  title={Multi-marginal {Gromov--W}asserstein transport and barycentres},
  author={Beier, Florian and Beinert, Robert and Steidl, Gabriele},
  journal={Information and Inference: A Journal of the IMA},
  volume={12},
  number={4},
  pages={2753--2781},
  year={2023},
  publisher={Oxford University Press}
}

@article{cover1967nearest,
  title={Nearest neighbor pattern classification},
  author={Cover, Thomas M and Hart, Peter E},
  journal={IEEE Transactions on Information Theory},
  volume={13},
  number={1},
  pages={21--27},
  year={1967},
  publisher={IEEE}
}

@inproceedings{nadjahi2021fast,
  title     = {Fast Approximation of the {Sliced-Wasserstein} distance Using concentration of random projections},
  author    = {Nadjahi, Karim and Vialard, Fran{\c{c}}ois-Xavier and Peyr{\'e}, Gabriel},
  booktitle = { International Conference on Machine Learning (ICML)},
  year      = {2021},
  volume    = {139},
  pages     = {7946--7956},
  series    = {Proceedings of Machine Learning Research},
  publisher = {PMLR},
  url       = {https://proceedings.mlr.press/v139/nadjahi21a.html}
}

@book{gray2004tubes,
  title        = {Tubes},
  author       = {Gray, Alfred},
  hideseries      = {Progress in Mathematics},
  seriesvolume       = {221},
  year         = {2004},
  publisher    = {Birkhäuser},
  address = {Basel},
  edition = {2nd},
  hideisbn         = {978-0-8176-4342-6},
  hidenote         = {See Chap.~1 for discussion and asymptotics of volume and area of $n$‑spheres}
}

@phdthesis{bonnotte2013unidimensional,
  author       = {Nicolas Bonnotte},
  title        = {Unidimensional and Evolution Methods for Optimal Transportation},
  school       = {Université Paris Sud–Paris XI},
  year         = {2013},
  address      = {Orsay, France},
  note         = {PhD Thesis}
}

@misc{sato2020fast,
  title        = {Fast and robust comparison of probability measures in heterogeneous spaces},
  author       = {Ryoma Sato and Marco Cuturi and Makoto Yamada and Hisashi Kashima},
  note      = {{arXiv:2002.01615}},
  year         = {2020},
  url          = {https://arxiv.org/abs/2002.01615}
}

@article{grohe2020graph,
  title={The graph isomorphism problem},
  author={Grohe, Martin and Schweitzer, Pascal},
  journal={Communications of the ACM},
  volume={63},
  number={11},
  pages={128--134},
  year={2020},
  publisher={ACM New York, NY, USA}
}

@article{weitkamp2024distribution,
  title={Distribution of distances based object matching: asymptotic inference},
  author={Weitkamp, Christoph Alexander and Proksch, Katharina and Tameling, Carla and Munk, Axel},
  journal={Journal of the American Statistical Association},
  volume={119},
  number={545},
  pages={538--551},
  year={2024},
  publisher={Taylor \& Francis}
}

@article{watts1998collective,
  title={{C}ollective dynamics of 'small-world' networks},
  author={Watts, Duncan J and Strogatz, Steven H},
  journal={Nature},
  volume={393},
  number={6684},
  pages={440--442},
  year={1998},
  publisher={Nature Publishing Group}
}

@book{bollobas1998random,
  title={{R}andom Graphs},
  author={Bollob{\'a}s, B{\'e}la},
  hidenumber={73},
  hideseries={Cambridge Studies in Advanced Mathematics},
  year={1998},
  publisher={Cambridge University Press},
  address={Cambridge},
  hidenote={This book provides a comprehensive treatment of random graphs, including regular random graphs.}
}

@article{flamary2021pot,
  author       = {Flamary, R{\'e}mi and Courty, Nicolas and Gramfort, Alexandre and Alaya, Mokhtar Z. and Boisbunon, Aur{\'e}lie and Chambon, Stanislas and Chapel, Laetitia and Corenflos, Adrien and Fatras, Kilian and Fournier, Nemo and Gautheron, L{\'e}o and Gayraud, Nathalie T. H. and Janati, Hicham and Rakotomamonjy, Alain and Redko, Ievgen and Rolet, Antoine and Schutz, Antony and Seguy, Vivien and Sutherland, Danica J. and Tavenard, Romain and Tong, Alexander and Vayer, Titouan},
  title        = {{POT}: {P}ython optimal transport},
  journal      = {Journal of Machine Learning Research},
  volume       = {22},
  number       = {78},
  pages        = {1--8},
  year         = {2021},
  url          = {http://jmlr.org/papers/v22/20-451.html},
  note         = {Software available at \url{https://pythonot.github.io/}},
}

@book{Mul98,
  title = {Analysis of spherical symmetries in {E}uclidean spaces},
  author = {M{\"u}ller, Claus},
  year = {1998},
  hideseries = {Applied Mathematical Sciences},
  publisher = {Springer},
  address = {New York},
  doi = {10.1007/978-1-4612-0581-4},
}

@article{beinert2023assignment,
  author  = {Beinert, Robert and Heiss, Cosmas and Steidl, Gabriele},
  title   = {On assignment problems related to {Gromov–-W}asserstein distances on the real line},
  journal = {SIAM Journal on Imaging Sciences},
  volume  = {16},
  number  = {2},
  pages   = {1028--1032},
  year    = {2023},
  doi     = {10.1137/22M1497808}
}

@inproceedings{nguyen2025sotdd,
  author    = {Nguyen, Khai and Nguyen, Hai and Pham, Tuan and Ho, Nhat},
  title     = {Lightspeed geometric dataset distance via sliced optimal transport},
  booktitle = {Proceedings of ICML'25},
  year      = {2025},
    publisher={OpenReview.net}
}

@inproceedings{redko2020coot,
  author    = {Vayer, Titouan and Redko, Ievgen and  Flamary, R\'{e}mi and Courty, Nicolas},
  title     = {{CO}‑optimal transport},
  booktitle = {Advances in Neural Information Processing Systems},
  year      = {2020},
  volume = {33},
  pages = {17559--17570},
  publisher = {Curran Associates}
}

@article{shervashidze2011wlkernel,
  author    = {Shervashidze, Nino and Schweitzer, Pascal and van Leeuwen, Erik Jan and Mehlhorn, Kurt and Borgwardt, Karsten},
  title     = {{Weisfeiler–L}ehman graph kernels},
  journal   = {Journal of Machine Learning Research},
  volume    = {12},
  pages     = {2539--2561},
  year      = {2011}
}

@inproceedings{vayer2019sgw,
  author  = {Vayer, Titouan and Chapel, Laetitia and Flamary, Rémi and Tavenard, Romain and Courty, Nicolas},
  title   = {Sliced {Gromov--W}asserstein},
  booktitle = {Advances in Neural Information Processing Systems},
  volume  = {32},
  pages   = {10525--10535},
  year    = {2019},
  publisher = {Curran Associates},
  note = {Correction: arxiv:1905.10124}
}

@misc{piening2025smw,
  author  = {Piening, Moritz and Beinert, Robert},
  title   = {Slicing the {G}aussian mixture {W}asserstein distance},
  note = {{arXiv:2504.08544}},
  year    = {2025}
}

@misc{piening2025slicedgw_aaai,
  author  = {Piening, Moritz and Beinert, Robert},
  title   = {A novel sliced fused {Gromov--Wasserstein} distance},
  note = {To appear in: Proceedings of the AAAI'26},
  year={2026}
}

@inproceedings{scetbon2022lowrankgw,
  author    = {Scetbon, Meyer and Peyré, Gabriel and Cuturi, Marco},
  title     = {Linear-time {Gromov-–W}asserstein distances using low‑rank couplings and costs},
  booktitle = {Proceedings of ICML'22},
  series    = {Proceedings of Machine Learning Research},
  volume    = {162},
  pages     = {19347--19365},
  year      = {2022},
  publisher = {PMLR},
  url       = {https://proceedings.mlr.press/v162/scetbon22b.html}
}

@inproceedings{scetbon2023unbalanced,
  title={Unbalanced low-rank optimal transport solvers},
  author={Scetbon, Meyer and Klein, Michal and Palla, Giovanni and Cuturi, Marco},
  booktitle={Advances in Neural Information Processing Systems},
  volume={36},
  pages={52312--52325},
  year={2023},
  publisher = {Curran Associates}
}

@inproceedings{peyre2016gromov,
  title={Gromov--{W}asserstein averaging of kernel and distance matrices},
  author={Peyr{\'e}, Gabriel and Cuturi, Marco and Solomon, Justin},
  booktitle={Proceedings of ICML'16},
  volume={48},
  series={Proceedings of Machine Learning Research},
  pages={2664--2672},
  year={2016},
  organization={PMLR}
}

@article{nguyen2023linearfused,
  title={On a linear fused {Gromov--W}asserstein distance for graph structured data},
  author={Nguyen, Dai Hai and Tsuda, Koji},
  journal={Pattern Recognition},
  volume={138},
  pages={109351},
  year={2023},
  publisher={Elsevier}
}

@inproceedings{chowdhury2021quantized,
  title={Quantized {Gromov--W}asserstein},
  author={Chowdhury, Samir and Miller, David and Needham, Tom},
  booktitle={Proceedings of ECML PKDD'21},
  pages={811--827},
  year={2021},
  organization={Springer}
}

@inproceedings{dukler2019wasserstein,
  title={Wasserstein of {W}asserstein loss for learning generative models},
  author={Dukler, Yonatan and Li, Wuchen and Lin, Alex and Mont{\'u}far, Guido},
  booktitle={Proceedings of ICML'19},
  pages={1716--1725},
  volume=97,
  series={Proceedings of Machine Learning Research},
  year={2019},
  publisher={PMLR}
}

@inproceedings{
bonet2025flowing,
title={Flowing datasets with {W}asserstein over {W}asserstein gradient flows},
author={Cl{\'e}ment Bonet and Christophe Vauthier and Anna Korba},
booktitle={Proceedings of ICML'25},
year={2025},
publisher={OpenReview.net}
}

@article{delon2020wasserstein,
  title={A {W}asserstein-type distance in the space of {G}aussian mixture models},
  author={Delon, Julie and Desolneux, Agnes},
  journal={SIAM Journal on Imaging Sciences},
  volume={13},
  number={2},
  pages={936--970},
  year={2020},
  publisher={SIAM}
}

@inproceedings{alvarez2020geometric,
  title={Geometric dataset distances via optimal transport},
  author={Alvarez-Melis, David and Fusi, Nicolo},
  booktitle={Advances in Neural Information Processing Systems},
  volume={33},
  pages={21428--21439},
  year={2020},
    publisher={Curran Associates}
}

@article{tanguy2025properties,
  title={Properties of discrete sliced {W}asserstein losses},
  author={Tanguy, Eloi and Flamary, R{\'e}mi and Delon, Julie},
  journal={Mathematics of Computation},
  volume={94},
  number={353},
  pages={1411--1465},
  year={2025}
}

@inproceedings{nguyen2024quasi,
  title     = {Quasi‑Monte Carlo for 3D Sliced Wasserstein},
  author    = {Khai Nguyen and Nicola Bariletto and Nhat Ho},
  booktitle = {Proceedings of the ICLR'24},
  year      = {2024},
  publisher={OpenReview.net}
}

@inproceedings{hertrich2025qmc_slicing,
  author       = {Johannes Hertrich and Tim Jahn and Michael Quellmalz},
  title        = {Fast summation of radial kernels via {QMC} slicing},
  booktitle    = {Proceedings of the ICLR'25},
  year         = {2025},
  publisher={OpenReview.net}
}

@inproceedings{hertrich2024generative,
  author    = {Johannes Hertrich and Christian Wald and Fabian Altekrüger and Paul Hagemann},
  title     = {Generative sliced {MMD} flows with {R}iesz kernels},
  booktitle = {Proceedings of the ICLR'24},
  year      = {2024},
  publisher={OpenReview.net}
}

@article{rux2025slicing,
  title={Slicing of radial functions: a dimension walk in the  {F}ourier space},
  author={Rux, Nicolaj and Quellmalz, Michael and Steidl, Gabriele},
  journal={Sampling Theory, Signal Processing, and Data Analysis},
  volume={23},
  number={1},
  pages={1--40},
  year={2025},
  publisher={Springer}
}

@article{quellmalz2023sliced_optimal_transport,
  author  = {Michael Quellmalz and Robert Beinert and Gabriele Steidl},
  title   = {Sliced optimal transport on the sphere},
  journal = {Inverse Problems},
  volume  = {39},
  number  = {4},
  pages   = {044003},
  year    = {2023},
  doi     = {10.1088/1361-6420/acc93d},
  url     = {https://doi.org/10.1088/1361-6420/acc93d}
}

@article{bonet2025sliced_hadamard,
  author    = {Cl{\'e}ment Bonet and Lucas Drumetz and Nicolas Courty},
  title     = {Sliced‑{W}asserstein distances and flows on {Cartan‑-Hadamard} manifolds},
  journal   = {Journal of Machine Learning Research},
  volume    = {26},
  number    = {32},
  pages     = {1--76},
  year      = {2025}
}

@article{kolouri2019generalized,
  title={Generalized sliced {W}asserstein distances},
  author={Kolouri, Soheil and Nadjahi, Kimia and Simsekli, Umut and Badeau, Roland and Rohde, Gustavo},
  journal={Advances in Neural Information Processing Systems},
  volume={32},
  year={2019},
  publisher={Curran Associates}
}

@inproceedings{beckmann2025max,
  title={Max-normalized {R}adon cumulative distribution transform for limited data classification},
  author={Beckmann, Matthias and Beinert, Robert and Bresch, Jonas},
  booktitle={International Conference on Scale Space and Variational Methods in Computer Vision},
  pages={241--254},
  year={2025},
  publisher={Springer}
}

@misc{beckmann2025normalized,
  title={Normalized {R}adon cumulative distribution transforms for invariance and robustness in optimal transport based image classification},
  author={Beckmann, Matthias and Beinert, Robert and Bresch, Jonas},
  note={{arXiv:2506.08761}},
  year={2025}
}

@inproceedings{uy2019revisiting_modelnet,
  title={Revisiting point cloud classification: {A} new benchmark dataset and classification model on real-world data},
  author={Uy, Mikaela Angelina and Pham, Quang-Hieu and Hua, Binh-Son and Nguyen, Thanh and Yeung, Sai-Kit},
  booktitle={Proceedings of the ICCV'19},
  pages={1588--1597},
  year={2019},
  publisher={IEEE}
}

\newpage
\onecolumn
\appendix
\section{Technical Appendix}

\subsection{Detailed Proofs of the Propositions}

 \begin{proof}[Proof of Prop.~\ref{prop:ftlb_prop}]
    (i)
    Lebesgue's dominated convergence theorem implies
    $g(x_n, \cdot)_{\sharp} \,\xi \to g(x, \cdot)_\sharp \, \xi$ weakly
    if $x_n \to x$ in $X$
    and
    $h(y_n, \cdot)_{\sharp} \,\upsilon \to h(y, \cdot)_\sharp \, \upsilon$ weakly
    if $y_n \to y$ in $Y$.
    Thus, the local distance distribution $\LD_{p}$ is continuous on $X \times Y$.
    Since the Euclidean distance is continuous as well,
    the infimum in \eqref{eq:FTLB} is attained
    \cite[Thm.~1.4]{Santambrogio2015}.
    
    (ii) Similarly to the relation between GW and TLB
    \cite[Prop.~6.3]{memoli2011gromov},
    for any $\gamma \in \Gamma(\pi_{X, \sharp} \, \xi, \pi_{Y, \sharp} \, \upsilon)$,
    we obtain
    \begin{align}
    \D_{p}(\gamma)
    &=
    \iint_{(X\times Y)^2}
    \lvert g(x,x') - h(y,y') \rvert^p
    \d\gamma(x', y') \d\gamma(x,y)
    \\
    &\geq
    \int_{X\times Y}
    \Biggl[
    \inf_{\gamma' \in \Gamma(\pi_{X,\sharp} \, \xi, \pi_{Y, \sharp} \, \upsilon)}
    \int_{X\times Y}
    \lvert g(x,x') - h(y,y') \rvert^p
    \d\gamma'(x', y') 
    \Biggr]\d\gamma(x,y)
    \\
    &
    \int_{X\times Y}
    \LD_p(x,y)
    \d\gamma(x,y)
    = 
    \tilde\D_p(\gamma).
    \end{align}
    
    (iii) 
    From the definition of FTLB, 
    part (ii),
    and the interpolation properties of FGW 
    \cite[Thm.~3]{vayer2020fused},
    we obtain
    \begin{equation}
        \alpha
        \W_p^p(\pi_{Z, \sharp} \, \xi, \pi_{Z', \sharp} \, \upsilon) 
        \leq
        \FTLB^p_{\alpha, p}(\XXS,\YYS) 
        \leq 
        \FGW^p_{\alpha, p}(\XXS,\YYS)
        \xlongrightarrow[\alpha \to 1]{}
        \W_p^p(\pi_{Z, \sharp} \, \xi, \pi_{Z', \sharp} \, \upsilon).
    \end{equation}
    Since the left-hand side has the same limit as the right-hand side,
    the statement is established.
    
    (iv)
    From the definition of FTLB
    and the compactness of $Z$,
    we deduce
    \begin{equation}{}
        (1 - \alpha) \TLB_p^p( \XX, \YY)
        \le
        \FTLB^p_{\alpha, p}(\XXS,\YYS)
        \le
        (1-\alpha) \TLB_p^p(\XX, \YY)
        + \alpha \, C,
    \end{equation}\normalsize
    where $C$ is an upper bound for the Wasserstein part,
    i.e.,
    $\int_{Z \times Z} \lVert z - z' \rVert^p \d \eta \le C$
    for all 
    $\eta \in\Gamma (\pi_{z,\sharp} \gamma_X, \pi_{z,\sharp} \gamma_Y)$.
    Here, the second inequality holds true,
    due to the gluing lemma \cite[Lem.~7.6]{villani2003topics}.
    More precisely,
    gluing the optimal TLP plan in $\Gamma(\pi_{X, \sharp}\,\xi, \pi_{Y,\sharp}\, \upsilon)$
    with $\xi$ and $\upsilon$,
    we obtain an FTLP plan in $\Pi(\xi, \upsilon)$.
    Taking the limit $\alpha \to 0$ establishes the statement.
     
    (v)
    Positivity and symmetry are clear. 
    For the triangle inequality,
    let $\gamma_i \in \Prob(X \times Y_i)$ be the optimal FTLB plan
    between $\XXS \coloneqq(X \times Z, g, \xi)$
    and $\YYS_i \coloneqq ( Y \times Z_i, h_i, \upsilon_i)$
    with $Z = Z_i$ and $i=1,2$.
    By the gluing lemma \cite[Lem.~7.6]{villani2003topics},
    there exists $\tilde \gamma \in \Prob((X \times Z) \times (Y_1 \times Z_1) \times (Y_2 \times Z_2))$
    with $\pi_{(X \times Z) \times(Y_i \times Z_i) \sharp} \, \tilde\gamma = \gamma_i$.
    Using the triangle inequality of the Wasserstein and Euclidean distances
    as well as the convexity of the power function,
    we obtain
    \begin{align}
        \FTLB^p_{\alpha,p}(\YYS_1, \YYS_2)
        &\le 
        \smashoperator{
        \int_{(X \times Z) \times(Y_1 \times Z_1) \times (Y_2 \times Z_2)}}
        (1 - \alpha) \LD_p^p(y_1, y_2)
        +
        \alpha \, \lVert z_1 - z_2 \rVert^p 
        \,\d \tilde \gamma((x,z), (y_1,z_1), (y_2, z_2))
        \\
        &\le 
        \smashoperator{
        \int_{(X \times Z) \times(Y_1 \times Z_1) \times (Y_2 \times Z_2)}}
        (1 - \alpha) \bigl(\LD_p(x,y_1) + \LD_p(x, y_2) \bigr)^p
        +
        \alpha \bigl(\lVert z - z_1 \rVert + \lVert z - z_2 \rVert \bigr)^p 
        \d \tilde \gamma((x,z), (y_1,z_1), (y_2, z_2))
        \\
        &\le
        2^{p-1}
        \Bigl(
            \smashoperator{
            \int_{(X \times Z) \times(Y_1 \times Z_1)}}
            (1 - \alpha) \, \LD_p^p(x,y_1) 
            +
            \alpha \, \lVert z - z_1 \rVert^p 
            \,\d \gamma_1((x,z), (y_1,z_1))
        \\[-15pt]
        &\hspace{150pt}+ 
            \smashoperator{
            \int_{(X \times Z) \times(Y_2 \times Z_2)}}
            (1 - \alpha) \, \LD_p^p(x,y_2) 
            +
            \alpha \, \lVert z - z_2 \rVert^p 
            \,\d \gamma_2((x,z), (y_2,z_2))
        \Bigr)  
        \\
        &= 
        2^{p-1} \bigl( \FTLB^p_{\alpha,p}(\XXS, \YYS_1) + \FTLB^p_{\alpha,p}(\XXS, \YYS_2) \bigr).
        \tag*{\qedhere}
    \end{align}
 \end{proof}

\begin{proof}[Proof of Prop. \ref{prop:prop_sftlb_properties}] 
    (i) Depending on the employed quadrature rule,
    the mapping $\xi \mapsto \xi^\Q_\alpha$ in \eqref{eq:approx-TLB}
    may not be injective;
    therefore,
    the SW distance in \eqref{eq:sftlb_def} only implies 
    the pseudo-metric properties of SFTLB.
    
    (ii) 
    By construction, 
    we have $\xi^\Q_\alpha \to \xi^\Q_1$ 
    and $\upsilon^\Q_\alpha \to \upsilon^\Q_1$
    weakly in $\Prob(\R^{r+d})$ as $\alpha \to 1$.
    Since SW is continuous with respect to the weak topology 
    \cite[Thm.~1]{nadjahi2020statistical},
    it follows $\SW_2(\xi^\Q_\alpha, \upsilon^\Q_\alpha) \to \SW_2(\xi^\Q_1, \upsilon^\Q_1)$.
    Observing that the support of $\xi^\Q_1$ and $\upsilon^\Q_1$ is contained in $\{0\} \times \R^d$,
    we deduce the assertion
    using Proposition~\ref{prop:dimensional_dependence_sliced}.
    
    (iii)
    By construction, 
    we have $\xi^\Q_\alpha \to \xi^\Q_0$ 
    and $\upsilon^\Q_\alpha \to \upsilon^\Q_0$
    weakly in $\Prob(\R^{r+d})$
    as $\alpha \to 0$.
    Since SW is continuous with respect to the weak topology 
    \cite[Thm.~1]{nadjahi2020statistical},
    it follows $\SW_2(\xi^\Q_\alpha, \upsilon^\Q_\alpha) \to \SW_2(\xi^\Q_0, \upsilon^\Q_0)$.
    Observing that the support of $\xi^\Q_0$ and $\upsilon^\Q_0$ is contained in $\R^r \times \{0\}$,
    we deduce the assertion
    using Proposition~\ref{prop:dimensional_dependence_sliced}.
\end{proof}

 \begin{proof}[Proof of Prop. \ref{prop:dimensional_dependence_sliced}]
    To establish the statement,
    we parametrize the integral of the SW distance
    by the first dimension.
    Following \cite[§1]{Mul98},
    every $\theta \in \Sph^{r+d-1}$ can be written as
    $(t, \sqrt{1 - t^2} \, \theta')$
    with $t \in [-1,1]$ and $\theta' \in \Sph^{r+d-2}$.
    For the integral,
    this means \smash{$\d\theta = (1 - t^2)^{\frac{r+d-3}{2}} \d t \d \theta'$}.
    Moreover,
    the considered measures $\zeta$ and $\zeta'$ are only supported
    on the last $d$ components of $\R^{r+d}$
    \revise{by assumption}.
    \revise{Now we denote}
    the projection to the last $r+d-1$ components
    by \revise{$P_1^\perp(x_1, \dots,x_{r+d}) \coloneqq (x_2, \dots,x_{r+d})$}.
    \revise{Since $\zeta$ and $\zeta'$ are supported on $0 \times \R^{r+d-1}$,
    the projection $P_1^\perp$ is actually a bijection on their supports.
    Therefore,
    the push-forwards
    $\tilde \zeta \coloneqq P^\perp_{1, \sharp} \, \zeta$
    and
    $\tilde \zeta' \coloneqq P^\perp_{1, \sharp} \, \zeta'$
    can be reverted,
    and we have 
    $\zeta = \delta_0 \otimes \tilde \zeta$
    and
    $\zeta' = \delta_0 \otimes \tilde \zeta'$.}
    Incorporating the parametrization 
    and the projection into the slicing,
    we obtain
    \begin{equation}
        \pi_{\theta,\sharp} \, \zeta
        =
        \bigl\langle\bigl(
        \begin{smallmatrix}
            t \\ \sqrt{1 - t^2} \theta'
        \end{smallmatrix}
        \bigr), \cdot \bigr\rangle_\sharp \, \zeta
        =
        \bigl(\sqrt{1 - t^2} \, \cdot \bigr)_\sharp 
        \bigl\langle \theta', \cdot \bigr\rangle_\sharp \, \tilde\zeta
        =
        \bigl(\sqrt{1 - t^2} \, \cdot \bigr)_\sharp
        \pi_{\theta', \sharp} \, \tilde \zeta
    \end{equation}
    Furthermore,
    we exploit that the Wasserstein distance is homogenous,
    i.e.,
    $\W_2^2( (\lambda \cdot)_\sharp \eta, (\lambda \cdot)_\sharp \eta') = \lambda^2 \W_2^2(\eta, \eta')$
    for all $\lambda \in \R$ and for all compactly supported $\eta, \eta'$;
    see \cite[Prop.~7.16.i]{villani2003topics}.
    With these preliminary considerations, 
    we deduce
    \begin{align}
        \SW_2^2(\zeta, \zeta')
        &=
        \frac{1}{\A(\Sph^{r+d-1})}
        \int_{\Sph^{r+d-1}}
        \W_2^2( \pi_{\theta, \sharp} \, \zeta, \pi_{\theta, \sharp} \, \zeta')
        \d\theta
        \\
        &=
        \frac{1}{\A(\Sph^{r+d-1})}
        \int_{\Sph^{r+d-2}}
        \int_{-1}^1
        \W_2^2( \pi_{(t, \sqrt{1 - t^2} \theta'), \sharp} \, \zeta, \pi_{(t, \sqrt{1 - t^2} \theta'), \sharp} \, \zeta')
        \, (1 - t^2)^{\frac{r+d-3}{2}}
        \d t \d \theta'
        \\
        &=
        \frac{1}{\A(\Sph^{r+d-1})}
        \int_{\Sph^{r+d-2}}
        \int_{-1}^1
        \W_2^2( \pi_{\theta', \sharp} \, \tilde\zeta, \pi_{\theta', \sharp} \, \tilde \zeta')
        \, (1 - t^2)^{\frac{r+d-1}{2}}
        \d t \d \theta'
        \\
        &=
        \frac{\A(\Sph^{r+d+1})}{\A(\Sph^{r+d})}
        \, \frac{\A(\Sph^{r+d-2})}{\A(\Sph^{r+d-1})}
        \, \SW_2^2 (\tilde \zeta, \tilde \zeta'),
    \end{align}
    where the integral with respect to $t$
    has been calculated in \cite[(§1.35)]{Mul98}.
    Repeating this argumentation $r$ times, 
    we obtain the assertion.
\end{proof}

\begin{proof}[Proof of Prop.~\ref{prop:metr-equi}]
    (i) The statement directly follows from
    Proposition~\ref{prop:ftlb_prop}
    in combination with
    \eqref{eq:exact_tlb_empirical}
    and \eqref{eq:sw_equivalnce}.
    
    (ii) The metric equivalence follows from 
    the exactness in \eqref{eq:exact_tlb_empirical},
    the corresponding equivalence \eqref{eq:sw_equivalnce}
    of the Wasserstein and SW distance, 
    and the definition of $\SFTLB$
    in \eqref{eq:sftlb_def}.
\end{proof}

\revise{
\begin{proof}[Extension of Prop.~\ref{prop:metr-equi}]
    In order to extent Prop.~\ref{prop:metr-equi}
    to the unbalanced case $n \ne m$,
    we have to choose a quadrature scheme
    such that 
    \eqref{eq:LD} holds precisely 
    instead of approximately,
    and thus
    \eqref{eq:exact_tlb_empirical} remains valid.
    For this,
    notice
    that the push-forwards 
    $g(x,\cdot)_\sharp \xi$ and $h(y,\cdot)_\sharp \upsilon$
    remain empirical measures
    and that
    the corresponding quantile functions
    $q_{g(x,\cdot)_\sharp \xi}$ and $q_{h(y,\cdot)_\sharp \upsilon}$
    for all $x \in X$ and $y \in Y$
    are step functions
    with exactly $n$ and $m$ equal length steps,
    respectively. 
    The difference $q_{g(x,\cdot)_\sharp \xi} - q_{h(y,\cdot)_\sharp \upsilon}$
    is also a step function,
    but with maybe unequal step lengths.
    The jumps in these differences, however, occur at the same positions
    for all $x \in X$ and $y \in Y$.
    Choosing the knots $s_k$ in the middle of the steps
    and the weights $w_k$ as the related step length,
    we obtain an exact quadrature scheme 
    with $r = n+m-1$ 
    for \eqref{eq:LD}.
\end{proof}}

%\newpage
\section{Further Experiments}

\subsection{Additional Runtime Experiment}
We extend the runtime experiment presented in Table~\ref{tab:time_tab}, where we compared our method with the exact OT solver from POT. Again, we employ random Euclidean distance matrices based on randomly initialized point clouds. Here, we present additional runtime experiments of SFTLB compared to FTLB with entropic regularization \cite{PeyreCuturi2019}. Since we observe slow POT runtimes for entropically regularized algorithms, we turn to another package. In particular, we employ the `\textit{geomloss}' package (Feydy et al. 2019) for entropic OT with strong  ($0.1$) and weak entropic regularization ($0.001$). Since the `\textit{geomloss}' package is incapable of handling $10 \,000 \times 10 \, 000$ matrices, we repeat the previous runtime experiment with graph sizes 50, 500, and 5\,000. Additionally, we present the SFTLB runtime with different configurations of the quadrature size $r$ and the projection number $L$. Again, we run our experiments on the CPU. Looking at the table, we again see a runtime advantage of SFTLB compared to FTLB. Additionally, we observe that $L$ has a strong impact on the runtime, whereas the impact of $r$ is \revise{more} subtle.
\begin{table}[ht]
    \centering
    \footnotesize
    \begin{tabular}{lccc}
    \toprule
    $n$ & 50 & 500  & 5\,000 \\
    \midrule
    Entropic FTLB (Regularization: $0.1$) & 0.00 $\pm$ 0.00 & 0.05 $\pm$ 0.01 & 2.95 $\pm$ 0.07\\
    Entropic FTLB (Regularization: $0.001$) & 0.01 $\pm$ 0.01 & 0.06 $\pm$ 0.01 & 4.14 $\pm$ 0.17\\
    SFTLB ($r=10$ \& $L=50$) & 0.00 $\pm$ 0.00 & 0.04 $\pm$ 0.01 & 0.30 $\pm$ 0.00 \\
    SFTLB ($r=10$ \& $L=1000$) & 0.00 $\pm$ 0.00 & 0.09 $\pm$ 0.01 & 0.70 $\pm$ 0.01 \\
    SFTLB ($r=100$ \& $L=50$) & 0.00 $\pm$ 0.00 & 0.05 $\pm$ 0.02 & 0.31$\pm$ 0.02 \\
    SFTLB ($r=100$ \& $L=1000$) & 0.01 $\pm$ 0.00 & 0.07 $\pm$ 0.01 & 0.74 $\pm$ 0.01 \\
    \bottomrule
    \end{tabular}
    \caption{Average computation time in seconds for 5 instances with graph sizes 50, 500, and 5\,000.}
    \label{tab:time_tab_ablation}
\end{table}
\subsection{Illustration of Hyperparameter Impact}
Next, we illustrate the impact of $L$ and $r$ visually. To illustrate the convergence of our Monte Carlo estimate, we plot
\begin{equation*}
    \frac{\frac1L \sum_{\ell=1}^L
     \W^2_2(\pi_{\theta_\ell,\sharp} \, \xi_\alpha^\Q, \pi_{\theta_\ell,\sharp} \, \upsilon^\Q_\alpha)}
     {\frac{1}{20 \, 000} \sum_{\ell=1}^{20 \, 000} 
     \W^2_2(\pi_{\theta_\ell,\sharp} \, \xi_\alpha^\Q, \pi_{\theta_\ell,\sharp} \, \upsilon^\Q_\alpha)}
\end{equation*}
for different choices of $L$, 10 different random Euclidean $50 \times 50$ distance matrices, 
and $r=10$ in Figure~\ref{fig:L_impact}. Each colored line corresponds to an input distance matrix pair. We see that after around 100 random projections, the estimate deviates by less than 10\% from the high-fidelity Monte Carlo estimate obtained with 20\,000 projections.
\\
Next, we use the same random matrices to display the impact of the quadrature grid $r$ on SFTLB. Here, it is important to notice that the choice of $r$ impacts the ground truth value of SFTLB—notably, a higher choice of $r$ leads to smaller values. This is illustrated for 10 random matrix pairs in Figure~\ref{fig:r_impact}, where we employ $L=1000$. Again, each colored line corresponds to a different quantile discretization for a given input distance matrix pair.
\\
\\
\revise{
\paragraph{Interplay of $r$, $L$ and the dimension:}
In this paragraph, we aim to provide additional insights into the dependence of $r$ and $L$ beyond the ablation study. For simplicity, we focus on the classical GW distance without any fused features.

From the theoretical point of view, it is important to note that the quadrature parameter $r$ corresponds 
to the dimension of $\xi^{\mathcal{Q}}, \nu^{\mathcal{Q}} \in \Prob(\R^r)$. 
Generally, slicing techniques are not always well-suited to high-dimensional settings since the required number of random projections can explode. In particular, the Monte Carlo estimate for the sliced Wasserstein distance has a convergence rate of the form
\begin{equation*}
\sqrt{\frac{1}{L}}\, C_{\text{MC}}(\xi^{\mathcal{Q}}, \nu^{\mathcal{Q}})
\quad\text{with}\quad
C_{\text{MC}}(\xi^{\mathcal{Q}}, \nu^{\mathcal{Q}})
\coloneqq
\operatorname{Var}_{\theta}^{1/2}\!\left[
\W_2^2\!\bigl(\pi_{\theta,\sharp}\,\xi^{\mathcal{Q}},\, \pi_{\theta,\sharp}\,\upsilon^{\mathcal{Q}}_{\alpha}\bigr)
\right],
\end{equation*}
where $\operatorname{Var}_\theta$ 
denotes the variance with respect to uniform spherical projections $\theta \sim \mathcal{U}(\Sph^{r-1})$, see \cite{nadjahi2020statistical}. 
Given arbitary measures on $\Prob(\R^r)$, we would expect the constant 
$C_{\text{MC}}(\xi^{\mathcal{Q}}, \nu^{\mathcal{Q}})$ 
to balloon for high dimensions $r$.
Practically, we can circumvent this issue by choosing a small $r \ll n, m$ as in our experiments, e.g., $r=10$ or $r=50$. 

However, even beyond this numerical workaround, we observe good classification results for $L \ll r$ in Table~\ref{tab:knn_shapes_ablation}. We hypothesize that this the effect occurs because $\xi^\Q, \nu^Q$ tend to concentrate on a `smaller' space, especially for large $r$. 
Particularly, $\xi^\Q, \nu^\Q$ are only supported on a cone $S_r = \{0 \leq x_1 \leq \ldots \leq x_r\}$.
Given fixed mm-spaces $\XX, \YY$, the constant $C_{\text{MC}}(\xi^{\mathcal{Q}}, \nu^{\mathcal{Q}})$ does hence not necessarily become larger for larger $r$.
}
\begin{figure}[th]
    \centering
    \begin{minipage}[t]{0.48\linewidth}
        \centering
        \includegraphics[width=\linewidth]{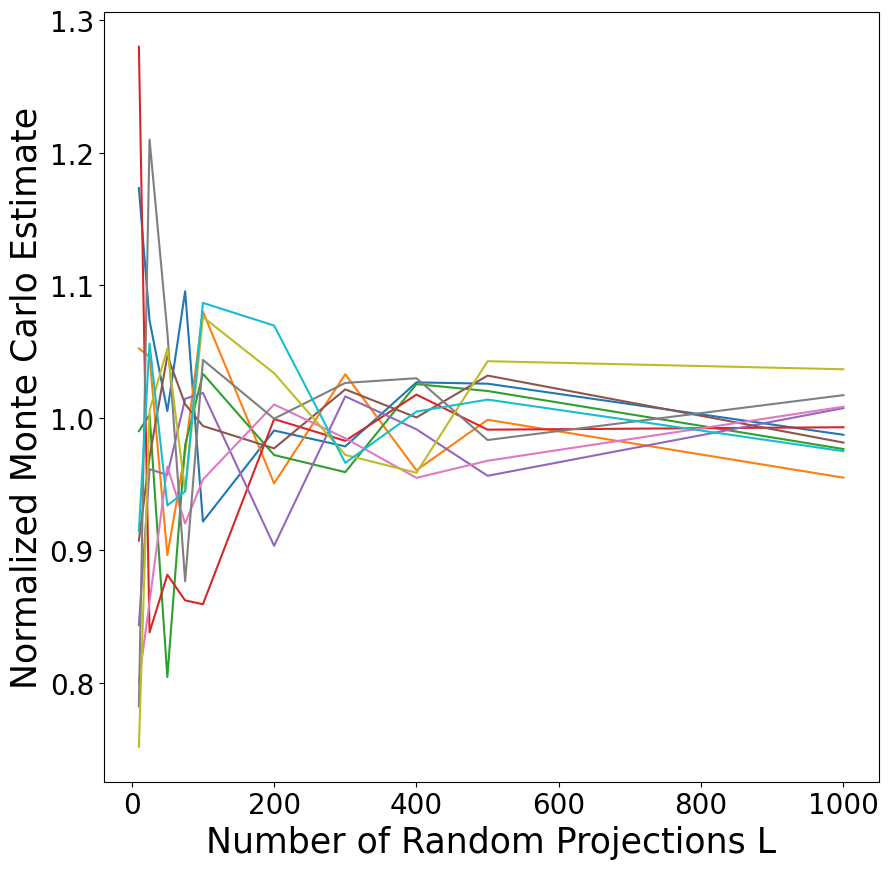}
        \caption{Impact of projection number $L$ on deviation of Monte Carlo SFTLB Estimate. \revise{Each line corresponds to multiple SFTLB estimates for a fixed simulated measure pair.}}
        \label{fig:L_impact}
    \end{minipage}
    \hfill
    \begin{minipage}[t]{0.48\linewidth}
        \centering
        \includegraphics[width=\linewidth]{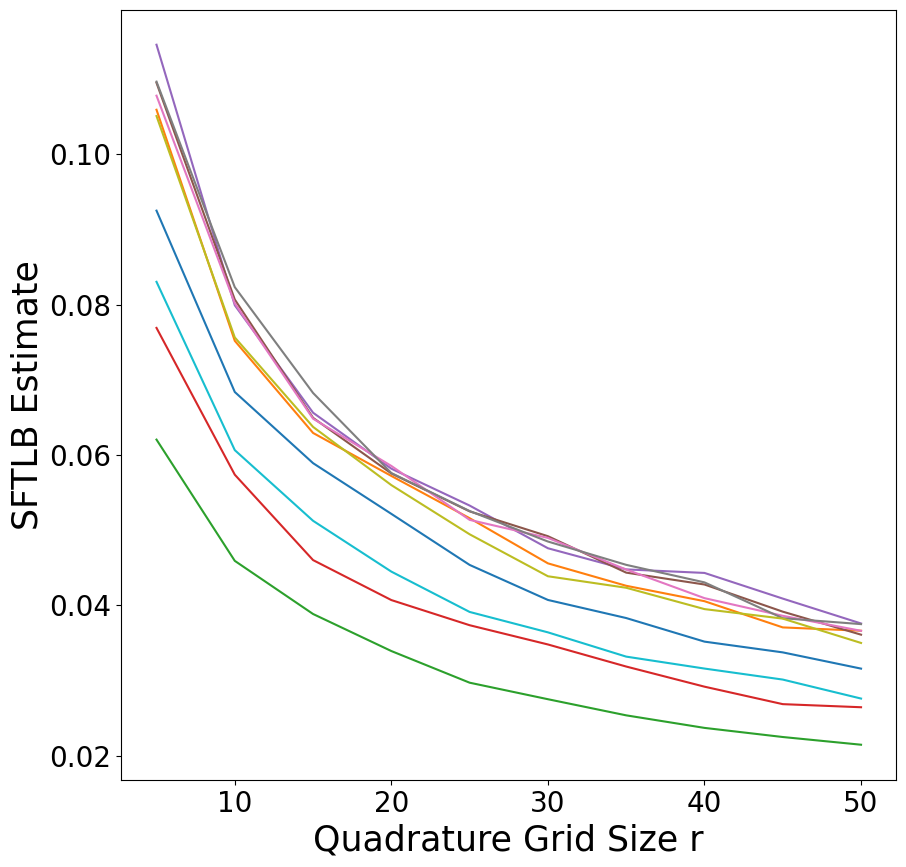}
        \caption{Large quadrature sizes $r$ lead to smaller SFTLB values for the same inputs.
        \revise{Each line corresponds to multiple SFTLB estimates for a fixed simulated measure pair.}}
        \label{fig:r_impact}
    \end{minipage}
\end{figure}
\subsection{Hyperparameter Ablation Study}

\sisetup{
  round-mode      = places,
  round-precision = 1
}
We repeat the '2D Shapes' and the 'FAUST-500' experiments presented in Table~\ref{tab:knn_shapes} for varying numbers of quadrature points $r=5, 25, 100$ and random projections $L=10, 100, 500$ and present the results in Table~\ref{tab:knn_shapes_ablation}. Higher choices of $r$ and $L$ increase accuracy slightly, but also decrease runtime. 
\revise{As for the hyperparameter $\alpha$ controlling the feature importance, we remark that its influence and optimal value depend on the application. It can be chosen based on domain knowledge or cross-validation strategies, see \cite{vayer2020fused}.}
\begin{table*}[ht]
  \centering
  \footnotesize                      % smaller font
  \setlength{\tabcolsep}{8pt}  % adjust spacing for fewer columns
  \renewcommand{\arraystretch}{1.1}
  \begin{tabular}{l
                  c c
                  c c}
    \toprule
    STLB Parameters
      & \multicolumn{2}{c}{2D shapes}
      & \multicolumn{2}{c}{FAUST-500} \\
    \cmidrule(lr){2-3} \cmidrule(lr){4-5}
      & acc. (\%) & time (ms)
      & acc. (\%) & time (ms) \\
    \midrule
    $r=5$ \& $L=10$     
      & 98.3$\pm$2.1  & 0.8
      & 33.9$\pm$5.1  & 4.7 \\
    $r=5$ \& $L=100$           
      & 97.9$\pm$2.2 & 1.2
      & 34.4$\pm$5.4  & 9.4 \\
    $r=5$ \& $L=500$           
      & 98.1$\pm$2.1 & 2.9
      & 34.4$\pm$5.4  & 17.1 \\
    $r=25$ \&  $L=10$  
      & 98.1$\pm$1.4  & 0.6
      & 35.3$\pm$5.4  & 4.6 \\
    $r=25$ \&  $L=100$         
      & 99.5$\pm$1.1  & 1.5
      & 35.1$\pm$5.4  & 9.9 \\
    $r=25$ \&  $L=500$         
      & 99.7$\pm$0.8  & 4.0
      & 35.4$\pm$5.5  & 17.6 \\
    $r=100$ \&  $L=10$  
      & 98.8$\pm$1.5  & 1.2
      & 33.2$\pm$5.0  & 5.1 \\
    $r=100$ \&  $L=100$         
      & 99.7$\pm$0.9  & 2.0
      & 35.6$\pm$5.4  & 9.9 \\
    $r=100$ \&  $L=500$         
      & 99.6$\pm$0.9  & 5.1
      & 35.1$\pm$5.4  & 17.7 \\
    \bottomrule
  \end{tabular}
  \caption{Ablation of KNN Shape classification: 
    Mean accuracy (acc.) with standard deviation
    and mean runtime (time).}
  \label{tab:knn_shapes_ablation}
\end{table*}

\revise{
\subsection{Implementation Details for LGW and qGW}
We employ the POT implementation for qGW \cite{chowdhury2021quantized},
where we set the parameters $n_1~=~n_2~=~25$. Moreover, we employ the official GitHub implementation (\url{https://github.com/Gorgotha/LGW/tree/master}) for LGW \cite{beier2022linear}. Here, we calculate the reference space via a GW barycenter of size 50 regarding all inputs for each experiment. Note that the runtime of this linearized method behaves differently from other methods since distance calculations become almost instantaneous after the computation of the reference space and the subsequent alignment of the input data. However, this preprocessing is very time-consuming and most of the time is spent on the barycenter calculation. Hence, we present the overall run time divided by the number of distances in Table~\ref{tab:knn_shapes}. 

We tried to choose suitable hyperparameters for both methods. 
Note that this choice involves a trade-off between quality and run time. As a result, the presented methods can achieve better classification results at the cost of a longer runtime. In particular, the reference space of our LGW implementation relies on a barycenter computation based on all inputs, whereas LGW allows accelerated computations by using a subsampled barycenter (at the cost of potential performance losses).

Lastly, we note that qGW offers the advantage of GW transport plans, and the LGW framework has the unique benefit of enabling highly efficient distance calculations for new and unseen data.

\subsection{Additional Large-Scale Experiment}
Adding to the experiments in Table~\ref{tab:knn_shapes}, we conduct an additional shape classification experiment on a larger dataset.
In particular, we focus on the test split of the ModelNet-10 dataset \cite{uy2019revisiting_modelnet}. This dataset split consists of 908 shapes from 10 different classes. We randomly downsample each shape to an Euclidean point cloud of size 50. 

Subsequently, we employ the same classification pipeline as in the other experiments presented in Table~\ref{tab:knn_shapes}. As a result, we get the following classification accuracies:
26.4\% $\pm$ 1.6\% for SLB, 27.9\% $\pm$ 1.7\% for TLB, 27.3\% $\pm$ 1.7\% for AE, 26.8\% $\pm$ 1.7\% for STLB and 27.3\% $\pm$ 1.7\% for GW. Again, STLB performs on par with other approaches.
}
\end{document}